\documentclass{article}
\usepackage{arxiv}
\usepackage[toc,page]{appendix}

\usepackage[usenames]{xcolor}

\usepackage{inputenc} %
\usepackage[T1]{fontenc}    %
\usepackage{hyperref}       %
\usepackage{url}            %
\usepackage{booktabs}       %
\usepackage{amsfonts}       %
\usepackage{nicefrac}       %
\usepackage{microtype}      %
\usepackage{lipsum}
\usepackage{mathtools}
\usepackage{cite}
\usepackage{amsmath}
\usepackage{amssymb}
\usepackage{amsthm}
\usepackage{algorithm,algcompatible}
\usepackage[toc,page]{appendix}

\usepackage{lineno}
\mathtoolsset{showonlyrefs=true}

\newtheorem{theorem}{Theorem}
\newtheorem{lemma}{Lemma}

\newtheorem{corollary}{Corollary}
\newtheorem{proposition}{Proposition}

\numberwithin{equation}{section}

\DeclareMathOperator{\sign}{sgn}

\DeclareMathOperator*{\esssup}{ess\,sup}

\usepackage{mathptmx}      %
\begin{document}

\title{Optimal Approximation Rates for Deep ReLU Neural Networks on Sobolev and Besov Spaces
}

\author{Jonathan W. Siegel \\
 Department of Mathematics\\
 Texas A\&M University\\
 College Station, TX 77843 \\
 \texttt{jwsiegel@tamu.edu} \\
}

\maketitle

\begin{abstract}
Let $\Omega = [0,1]^d$ be the unit cube in $\mathbb{R}^d$. We study the problem of how efficiently, in terms of the number of parameters, deep neural networks with the ReLU activation function can approximate functions in the Sobolev spaces $W^s(L_q(\Omega))$ and Besov spaces $B^s_r(L_q(\Omega))$, with error measured in the $L_p(\Omega)$ norm. This problem is important when studying the application of neural networks in a variety of fields, including scientific computing and signal processing, and has previously been solved only when $p=q=\infty$. Our contribution is to provide a complete solution for all $1\leq p,q\leq \infty$ and $s > 0$ for which the corresponding Sobolev or Besov space compactly embeds into $L_p$. The key technical tool is a novel bit-extraction technique which gives an optimal encoding of sparse vectors. This enables us to obtain sharp upper bounds in the non-linear regime where $p > q$. We also provide a novel method for deriving $L_p$-approximation lower bounds based upon VC-dimension when $p < \infty$. 
Our results show that very deep ReLU networks significantly outperform classical methods of approximation in terms of the number of parameters, but that this comes at the cost of parameters which are not encodable.
\end{abstract}

\section{Introduction}
Deep neural networks have achieved remarkable success in both machine learning \cite{lecun2015deep} and scientific computing \cite{raissi2019physics,han2018solving}. However, a precise theoretical understanding of why deep neural networks are so powerful has not been attained and is an active area of research. An important part of this theory is the study of the approximation properties of deep neural networks, i.e. to understand how efficiently a given class of functions can be approximated using deep neural networks. In this work, we solve this problem for the class of deep ReLU neural networks \cite{nair2010rectified} when approximating functions lying in a Sobolev or Besov space with error measured in the $L_p$-norm. We remark that the ReLU activation functions is very widely used and is a major driver of many recent breakthroughs in deep learning \cite{goodfellow2016deep,lecun2015deep,nair2010rectified}.

Let us begin by giving a description of the Sobolev function classes, which are widely used in the theory of solutions to partial differential equations (PDEs) \cite{evans2010partial}, and the Besov function classes, which are widely used in approximation theory \cite{devore1993constructive}, statistics \cite{donoho1995adapting,donoho1998minimax}, and signal processing \cite{devore1992image}.

Let $\Omega\subset \mathbb{R}^d$ be a bounded domain, which we take to be the unit cube $\Omega = [0,1]^d$ in the following. Due to a variety of extension theorems for Sobolev and Besov spaces (see for instance \cite{evans2010partial,di2012hitchhikers,devore1993constructive,whitney1934analytic}), this is not a significant restriction and our results will apply to many other sufficiently well-behaved domains. We denote by $L_p(\Omega)$ the set of functions $f$ for which the $L_p$-norm on $\Omega$ is finite, i.e.
\begin{equation}
	\|f\|_{L_p(\Omega)} = \left(\int_{\Omega}|f(x)|^pdx\right)^{1/p} < \infty.
\end{equation}
When $p=\infty$, this becomes $\|f\|_{L_\infty(\Omega)} = \esssup_{x\in \Omega} |f(x)|$. Suppose that $s > 0$ is a positive integer. Then $f\in W^s(L_q(\Omega))$ is in the Sobolev space (see \cite{demengel2012functional}, Chapter 2 for instance) with $s$ derivatives in $L_q$ if $f$ has weak derivatives of order $s$ and
\begin{equation}
	\|f\|^q_{W^s(L_q(\Omega))} := \|f\|^q_{L_q(\Omega)} + \sum_{|\alpha| = k} \|D^\alpha f\|^q_{L_q(\Omega)} < \infty.
\end{equation}
Here $\alpha = (\alpha_i)_{i=1}^d$ with $\alpha_i\in \mathbb{Z}_{\geq 0}$ is a multi-index and $|\alpha| = \sum_{i=1}^d \alpha_i$ is the total degree. The $W^s(L_q(\Omega))$ semi-norm is defined by
\begin{equation}\label{sobolev-semi-norm-definition}
    |f|_{W^s(L_q(\Omega))} := \left(\sum_{|\alpha| = k} \|D^\alpha f\|^q_{L_q(\Omega)}\right)^{1/q},
\end{equation}
and the standard modifications are made when $q = \infty$.

When $s>0$ is not an integer, we write $s = k + \theta$ with $k \geq 0$ an integer and $\theta\in (0,1)$. The Sobolev semi-norm is defined by (see \cite{demengel2012functional} Chapter 4 or \cite{di2012hitchhikers} Chapter 1 for instance)
\begin{equation}\label{sobolev-semi-norm-definition-fractional}
    |f|^q_{W^s(L_q(\Omega))} := \int_{\Omega\times \Omega} \frac{|D^\alpha f(x) - D^\alpha f(y)|^q}{|x - y|^{d+\theta q}}dxdy
\end{equation}
when $1\leq q < \infty$ and
\begin{equation}
    |f|_{W^s(L_\infty(\Omega))} := \sup_{|\alpha| = k}\sup_{x,y\in \Omega}\frac{|D^\alpha f(x) - D^\alpha f(y)|}{|x - y|^\theta}.
\end{equation}
We define the Sobolev norm by
\begin{equation}
	\|f\|^q_{W^s(L_q(\Omega))} := \|f\|^q_{L_q(\Omega)} + |f|^q_{W^s(L_q(\Omega))},
\end{equation}
with the usual modification when $q = \infty$. We remark that in the case of non-integral $s$ these spaces are also called Sobolev-Slobodeckij spaces. Sobolev spaces are widely used in PDE theory and a priori estimates for PDE solutions are often given in terms of Sobolev norms \cite{evans2010partial}. For applications of neural networks to scientific computing it is thus important to understand how efficiently neural networks can approximate functions from $W^s(L_q(\Omega))$.

Next, we consider the Besov spaces, which we define in terms of moduli of smoothness. Given a function $f\in L_q(\Omega)$ and an integer $k$, the $k$-th order modulus of smoothness of $f$ is given by
\begin{equation}\label{definition-of-modulus-of-smoothness}
    \omega_k(f,t)_q = \sup_{|h|\leq t}\|\Delta^k_h f\|_{L_q(\Omega_{kh})},
\end{equation}
where $h\in \mathbb{R}^d$, the $k$-th order finite difference $\Delta^k_h$ is defined by
$$
    \Delta^k_h f(x) = \sum_{j=0}^k (-1)^j\binom{k}{j} f(x+jh),
$$
and the $L_q$ norm is taken over the set $\Omega_{kh} := \{x\in \Omega,~x + kh\in \Omega\}$ to guarantee that all terms of the finite difference are contained in the domain $\Omega$. Fix an integer $k > s$. The Besov space $B^s_r(L_q(\Omega))$ is defined via the norm
\begin{equation}
    \|f\|_{B^s_r(L_q(\Omega))} := \|f\|_{L_q(\Omega)} + |f|_{B^s_r(L_q(\Omega))},
\end{equation}
with Besov semi-norm given by
\begin{equation}
    |f|_{B^s_r(L_q(\Omega))} := \left(\int_0^\infty \frac{\omega_k(f,t)_q^r}{t^{sr+1}}dt\right)^{1/r}
\end{equation}
when $r < \infty$ and by
\begin{equation}
    |f|_{B^s_\infty(L_q(\Omega))} := \sup_{t > 0} t^{-s}\omega_k(f,t)_q,
\end{equation}
when $r = \infty$. It can be shown that different choices of $k > s$ result in equivalent norms \cite{devore1993constructive}.  One can think of the Besov space $B^s_r(L_q(\Omega))$ roughly as being a space of functions with $s$ derivatives lying in $L_q$, similar to the Sobolev space $W^s(L_q(\Omega))$, with the additional index $r$ providing a finer gradation. Indeed, a variety of embedding and interpolation results relating Besov spaces and Sobolev spaces are known (see for instance \cite{devore1988interpolation,devore1984maximal,yuan2010morrey,kufner1977function}).

Besov spaces are central objects in approximation theory due to their close connection with approximation by trigonometric polynomials (on the circle) and splines \cite{devore1993constructive,petrushev1988direct}. In fact, there are equivalent definitions of the Besov semi-norms in terms of approximation error by trigonometric polynomials and splines. They are also closely connected to the theory of wavelets \cite{daubechies1992ten}, and one can give equivalent definitions of the Besov norms in terms of the wavelet coefficients of $f$  as well \cite{devore1992image}. For this reason, Besov spaces play an important role in signal processing \cite{chambolle1998nonlinear,donoho1998data} and statistical recovery of functions from point samples \cite{donoho1995adapting,donoho1998minimax}, for instance.

Our goal is to study the approximation of Sobolev and Besov functions by neural networks. One of the most important classes of neural networks are deep ReLU neural networks, which we define as follows. We use the notation $A_{\textbf{W},b}$ to denote the affine map with weight matrix $\textbf{W}$ and offset, or bias, $b$, i.e.
\begin{equation}
    A_{\textbf{W},b}(x) = \textbf{W}x + b.
\end{equation}
When the weight matrix $\textbf{W}$ is an $k\times n$ and the bias $b\in \mathbb{R}^k$, the function $A_{\textbf{W},b}:\mathbb{R}^n\rightarrow \mathbb{R}^k$ maps $\mathbb{R}^n$ to $\mathbb{R}^k$. Let $\sigma$ denote the ReLU activation function \cite{nair2010rectified}, specifically
\begin{equation}
    \sigma(x) = \begin{cases}
    0 & x < 0\\
    x & x \geq 0.
    \end{cases}
\end{equation}
The ReLU activation function $\sigma$ has become ubiquitous in deep learning in the last decade and is used in most state-of-the-art architectures. Since $\sigma$ is continuous and piecewise linear, it also has the nice theoretical property that neural networks with ReLU activation function represent continuous piecewise linear functions. This property has been extensively studied in the computer science literature \cite{arora2018understanding,wang2005generalization,serra2018bounding,hanin2019complexity} and has been connected with traditional linear finite element methods \cite{he2020relu}.

When $x\in \mathbb{R}^n$, we write $\sigma(x)$ to denote the application of the activation function $\sigma$ to each component of $x$ separately, i.e. $\sigma(x)_i = \sigma(x_i)$. The set of deep ReLU neural networks with width $W$ and depth $L$ mapping $\mathbb{R}^d$ to $\mathbb{R}^k$ is given by
\begin{equation}
    \Upsilon^{W,L}(\mathbb{R}^d,\mathbb{R}^k) := \{A_{\textbf{W}_L,b_L} \circ \sigma \circ A_{\textbf{W}_{L-1},b_{L-1}} \circ \sigma \circ \cdots \circ \sigma \circ A_{\textbf{W}_1,b_1} \circ \sigma \circ A_{\textbf{W}_0,b_0}\},
\end{equation}
where the weight matrices satisfy $\textbf{W}_L\in \mathbb{R}^{k\times W}$, $\textbf{W}_0\in \mathbb{R}^{W\times d}$, and $\textbf{W}_1,...,\textbf{W}_{L-1}\in \mathbb{R}^{W\times W}$, and the biases satisfy $b_0,...,b_{L-1}\in \mathbb{R}^W$ and $b_L\in \mathbb{R}^k$. Notice that our definition of width does not include the input and output dimensions and only includes the intermediate layers. When the depth $L = 0$, i.e. when the network is an affine function, there are no intermediate layers and the width is undefined, in this case we write $\Upsilon^0(\mathbb{R}^d,\mathbb{R}^k)$.
We also use the notation
\begin{equation}
    \Upsilon^{W,L}(\mathbb{R}^d) := \Upsilon^{W,L}(\mathbb{R}^d,\mathbb{R})
\end{equation}
to denote the set of ReLU deep neural networks with width $W$, depth $L$ which represent scalar functions. We note that our notation only allows neural networks with fixed width. We do this to avoid excessively cumbersome notation. We remark that the dimension of any hidden layer can naturally be expanded and thus any fully connected network can be made to have a fixed width.

The problem we study in this work is to determine optimal $L_p$-approximation rates
\begin{equation}\label{approximation-rate-equation}
	\sup_{\|f\|_{W^s(L_q(\Omega))} \leq 1}\left(\inf_{f_L\in \Upsilon^{W,L}(\mathbb{R}^d)} \|f - f_L\|_{L_p(\Omega)}\right)~~\text{and}~~\sup_{\|f\|_{B^s_r(L_q(\Omega))} \leq 1}\left(\inf_{f_L\in \Upsilon^{W,L}(\mathbb{R}^d)} \|f - f_L\|_{L_p(\Omega)}\right)
\end{equation}
for the class of Sobolev and Besov functions using very deep ReLU networks, i.e. using networks with a fixed (large enough) width $W$ and depth $L\rightarrow \infty$. We will prove that this gives the best possible approximation rate in terms of the number of parameters. One can more generally consider approximation error in terms of both the width $W$ and depth $L$ simultaneously \cite{shen2022optimal}, but we leave this more general analysis as future work.

This problem has been previously solved (up to logarithmic factors) in the case where $p = q = \infty$, where the optimal rate is given by
\begin{equation}\label{optimal-rate-l-infty}
	\inf_{f_L\in \Upsilon^{W,L}(\mathbb{R}^d)} \|f - f_L\|_{L_\infty(\Omega)} \leq C\|f\|_{W^s(L_\infty(\Omega))}L^{-2s/d}
\end{equation}
for a sufficiently large but fixed width $W$. Specifically, this result was obtained for $0 < s \leq 1$ in \cite{yarotsky2018optimal} and for all $s > 0$ (up to logarithmic factors) in \cite{lu2021deep}. An analogous result also holds for $B^s_r(L_\infty(\Omega))$ for $1\leq r \leq \infty$. Further, the best rate when both the width and depth vary (which generalizes \eqref{optimal-rate-l-infty}) has been obtained in \cite{shen2022optimal}.

The method of proof in these cases uses the bit-extraction technique introduced in \cite{bartlett1998almost} and developed further in \cite{bartlett2019nearly} to represent piecewise polynomial functions on a fixed regular grid with $N$ cells using only $O(\sqrt{N})$ parameters. This enables an approximation rate of $CN^{-2s/d}$ in terms of the number of parameters $N$, which is significantly faster than traditional methods of approximation. This phenomenon has been called the \textit{super-convergence} of deep ReLU networks \cite{yarotsky2018optimal,shen2022optimal,devore2021neural,daubechies2022nonlinear}. The super-convergence has a limit, however, and the rate \eqref{optimal-rate-l-infty} is shown to be optimal using the VC-dimension of deep ReLU neural networks \cite{yarotsky2018optimal,shen2022optimal,bartlett2019nearly}.

In this work, we generalize this analysis to determine the optimal approximation rates \eqref{approximation-rate-equation} for all $1\leq p,q\leq \infty$ and $s > 0$, i.e. to the approximation of any Sobolev or Besov class in $L_p(\Omega)$, with the exception of the Sobolev embedding endpoint (described below). This was posed as a significant open problem in \cite{devore2021neural}. We remark that the existing upper bounds in $L_\infty$ clearly imply corresponding upper bounds in $L_p$ for $p < \infty$. The key problem lies in extending the upper bounds to that case where $q < \infty$, in which case we must approximate a larger function class. A further problem is the extension of the lower bounds to the case $p < \infty$, in which we are measuring error in a weaker norm.

A necessary condition that we have any approximation rate in \eqref{approximation-rate-equation} at all is for the Sobolev space $W^s(L_q(\Omega))$ or Besov space $B^s_r(L_q(\Omega)$ to be contained in $L_p$, i.e. $W^s(L_q(\Omega)), B^s_r(L_q(\Omega)\subset L_p(\Omega)$. Indeed, any deep ReLU neural network represents a continuous function and so if $f\notin L_p(\Omega)$ it cannot be approximated at all by deep ReLU networks. We will in fact consider the case where we have a compact embedding $W^s(L_q(\Omega)),B^s_r(L_q(\Omega)\subset\subset L_p(\Omega)$. Here the symbol $A \subset\subset B$ for two Banach spaces $A$ and $B$ means that $A$ is contained in $B$ and the unit ball of $A$ is a compact subset of $B$.
This compact embedding is guaranteed for both Besov and Sobolev spaces by the strict Sobolev embedding condition
\begin{equation}\label{strict-sobolev-embedding-condition-intro}
	\frac{1}{q} - \frac{1}{p} - \frac{s}{d} < 0.
\end{equation}
We determine the optimal rates in \eqref{approximation-rate-equation} under this condition. Specifically, we prove the following Theorems. The first two give an upper bound on the approximation rate by deep ReLU networks on Sobolev and Besov spaces, respectively.
\begin{theorem}\label{deep-network-upper-bound-theorem}
    Let $\Omega = [0,1]^d$ be the unit cube in $\mathbb{R}^d$ and let $0<s<\infty$ and $1\leq p,q \leq \infty$. Assume that $\frac{1}{q} - \frac{1}{p} < \frac{s}{d}$, which guarantees that we have the compact embedding
    \begin{equation}
        W^s(L_q(\Omega)) \subset\subset L_p(\Omega).
    \end{equation}
    Then we have that
    \begin{equation}
        \inf_{f_L\in \Upsilon^{25d+31,L}(\mathbb{R}^d)} \|f - f_L\|_{L_p(\Omega)} \leq C\|f\|_{W^s(L_q(\Omega))}L^{-2s/d}
    \end{equation}
    for a constant $C:=C(s,q,p,d) < \infty$.
\end{theorem}
\begin{theorem}\label{deep-network-upper-bound-theorem-besov}
    Let $\Omega = [0,1]^d$ be the unit cube in $\mathbb{R}^d$ and let $0<s<\infty$ and $1\leq r,p,q \leq \infty$. Assume that $\frac{1}{q} - \frac{1}{p} < \frac{s}{d}$, which guarantees that we have the compact embedding
    \begin{equation}
        B^s_r(L_q(\Omega)) \subset\subset L_p(\Omega).
    \end{equation}
    Then we have that
    \begin{equation}
        \inf_{f_L\in \Upsilon^{25d+31,L}(\mathbb{R}^d)} \|f - f_L\|_{L_p(\Omega)} \leq C\|f\|_{B^s_r(L_q(\Omega))}L^{-2s/d}
    \end{equation}
    for a constant $C:=C(s,r,q,p,d) < \infty$.
\end{theorem}
We remark that the constant in Theorem \ref{deep-network-upper-bound-theorem-besov} can be chosen uniformly in $r$. Note that the width $W = 25d + 31$ of our networks are fixed as $L\rightarrow \infty$, but scale linearly with the input dimension $d$. We remark that a linear scaling with the input dimension is necessary since if $d \geq W$, then the set of deep ReLU networks is known to not be dense in $C(\Omega)$ \cite{hanin2019universal}. The next Theorem gives a lower bound which shows that the rates in Theorems \ref{deep-network-upper-bound-theorem} and \ref{deep-network-upper-bound-theorem-besov} are sharp in terms of the number of parameters.
\begin{theorem}\label{deep-network-lower-bound-theorem}
	Let $r,p,q \geq 1$ and $s > 0$, $\Omega = [0,1]^d$ be the unit cube, and $W,L\geq 1$ be integers. Then there exists an $f$ with $\|f\|_{W^s(L_q(\Omega))} \leq 1$ and $\|f\|_{B^s_r(L_q(\Omega))} \leq 1$ such that
	\begin{equation}
		\inf_{f_{W,L}\in \Upsilon^{W,L}(\mathbb{R}^d)} \|f - f_{W,L}\|_{L_p(\Omega)} \geq C(p,d,s)\min\{W^2L^2\log(WL),W^3L^2\}^{-s/d}.
	\end{equation}
\end{theorem}
We remark that if the embedding condition \eqref{strict-sobolev-embedding-condition-intro} strictly fails, then a simply scaling argument shows that $W^s(L_q(\Omega)), B^s_r(L_q(\Omega)) \nsubseteq L_p(\Omega)$ and we cannot get any approximation rate. On the boundary where the embedding condition \eqref{strict-sobolev-embedding-condition-intro} holds with equality it is not a priori clear whether one has an embedding or not (this depends on the precise values of $s,p,q$ and $r$). Consequently this boundary case is much more subtle and we leave this for future work.

The key technical difficulty in proving Theorem \ref{deep-network-upper-bound-theorem} is to deal with the case when $p > q$, i.e. when the target function's (weak) derivatives are in a weaker norm than the error. Classical methods of approximation using piecewise polynomials or wavelets can attain an approximation rate of $CN^{-s/d}$ with $N$ wavelet coefficients or piecewise polynomials with $N$ pieces. When $p \leq q$ this rate can be achieved by linear methods, while for $p > q$ nonlinear, i.e. adaptive, methods are required. For the precise details of this theory, see for instance \cite{devore1993constructive,lorentz1996constructive,devore1998nonlinear}.

Thus, in the linear regime where $p\leq q$ we can use piecewise polynomials on a fixed \textit{uniform} grid to approximate $f$, while in the non-linear regime we need to use piecewise polynomials on an adaptive (i.e. depending upon $f$) \textit{non-uniform} grid. This greatly complicates the bit-extraction technique used to obtain super-convergence, since the methods in \cite{yarotsky2018optimal,shen2022optimal,shijun2021deep} are only applicable to regular grids. The tool that we develop to overcome this difficulty is a novel bit-extraction technique, presented in Theorem \ref{sparse-approximation-theorem}, which optimally encodes \textit{sparse} vectors using deep ReLU networks. Specifically, suppose that $\textbf{x}\in \mathbb{Z}^N$ is an $N$-dimensional integer vector with $\ell^1$-norm bounded by
\begin{equation}
    \|\textbf{x}\|_{\ell^1} \leq M.
\end{equation}
In Theorem \ref{sparse-approximation-theorem} we give (depending upon $N$ and $M$) a deep ReLU neural network construction which optimally encodes $\textbf{x}$.

We remark, however, that super-convergence comes at the cost of parameters which are non-encodable, i.e. cannot be encoded using a fixed number of bits, and this makes the numerical realization of this approximation rate inherently unstable. In order to better understand this, we recall the notion of metric entropy first introduced by Kolmogorov. The metric entropy numbers $\epsilon_N(A)$ of a set $A\subset X$ in a Banach space $X$ are given by (see for instance \cite{lorentz1996constructive}, Chapter 15)
\begin{equation}
 \epsilon_N(A)_H = \inf\{\epsilon > 0:~\text{$A$ is covered by $2^N$ balls of radius $\epsilon$}\}.
\end{equation}
An encodable approximation method consists of two maps, an encoding map $E:A\rightarrow \{0,1\}^N$ mapping the class $A$ to a bit-string of length $N$, and a decoding map $D:\{0,1\}^N\rightarrow X$ which maps each bit-string to an element of $X$. This reflects the fact that any method which is implemented on a classical computer must ultimately encode all parameters using some number of bits. The metric entropy numbers give the minimal reconstruction error of the best possible encoding scheme.

Let $U^s(L_q(\Omega)) := \{f:~\|f\|_{W^s(L_q(\Omega))} \leq 1\}$ denote the unit ball of the Sobolev space $W^s(L_q(\Omega))$. The metric entropy of this function class is given by
\begin{equation}
	\epsilon_N(U^s(L_q(\Omega)))_{L_p(\Omega)} \eqsim N^{-s/d}
\end{equation}
whenever the Sobolev embedding condition \eqref{strict-sobolev-embedding-condition-intro} is strictly satisfied. This is known as the Birman-Solomyak Theorem \cite{birman1967piecewise}. The same asymptotics for the metric entropy also hold for the unit balls in the Besov spaces $B^s_r(L_q(\Omega))$ if the compact embedding condition \eqref{strict-sobolev-embedding-condition} is satisfied. So the approximation rates in Theorems \ref{deep-network-upper-bound-theorem} and \ref{deep-network-upper-bound-theorem-besov} are significantly smaller than the metric entropy of the Sobolev and Besov classes. This manifests itself in the fact that in the construction of the upper bounds in Theorems \ref{deep-network-upper-bound-theorem} and \ref{deep-network-upper-bound-theorem-besov} the parameters of the neural network cannot be specified using a fixed number of bits, but rather need to be specified to higher and higher accuracy as the network grows \cite{yarotsky2020phase}, which is a direct consequence of the bit-extraction technique.

Concerning the lower bounds, the key difficulty in proving Theorem \ref{deep-network-lower-bound-theorem} is to extend the VC-dimension arguments used to obtain lower bounds when the error is measured in $L_\infty$ to the case when the error is measured in the weaker norm $L_p$ for $p < \infty$. We do this by proving Theorem \ref{main-theorem}, which gives a general lower bound for $L_p$-approximation of Sobolev spaces by classes with bounded VC dimension. We have recently learned of a different approach to obtaining $L_p$ lower bounds using VC-dimension \cite{achour2022general}, which is more generally applicable but introduces additional logarithmic factors in the lower bound.

We remark that there are other results in the literature which obtain approximation rates for deep ReLU networks on Sobolev spaces, but which do not achieve superconvergence, i.e. for which the approximation rate is only $CN^{-s/d}$ (up to logarithmic factors), where $N$ is the number of parameters \cite{yarotsky2017error,guhring2020error}. In addition, the approximation of other novel function classes (other than Sobolev spaces, which suffer the curse of dimensionality) by neural networks has been extensively studied recently, see for instance \cite{daubechies2022nonlinear,petersen2018optimal,daubechies2022neural,siegel2020approximation,siegel2022high,siegel2022sharp,bach2017breaking,klusowski2018approximation}.

Finally, we remark that although we focus on the ReLU activation function due to its popularity and to simplify the presentation, our results also apply to more general activation functions as well. Specifically, the lower bounds in Theorem \ref{deep-network-lower-bound-theorem} based upon VC-dimension hold for any piecewise polynomial activation function. The upper bounds in Theorems \ref{deep-network-upper-bound-theorem} and \ref{deep-network-upper-bound-theorem-besov} hold as long as we can approximate the ReLU to arbitrary accuracy on compact subsets (i.e. finite intervals) using a network with a fixed size. Using finite differences this can be done for the ReLU$^k$ activation functions defined by
\begin{equation}
        \sigma_k(x) = \begin{cases}
    0 & x < 0\\
    x^k & x \geq 0
    \end{cases}
\end{equation}
when $k \geq 1$ for instance. In fact, a similar construction using finite differences can approximate the ReLU as long as the activation function is a continuous piecewise polynomial which is not a polynomial.

The rest of the paper is organized as follows. First, in Section \ref{basic-network-constructions-section} we describe a variety of deep ReLU neural network constructions which will be used to prove Theorem \ref{deep-network-upper-bound-theorem}. Many of these constructions are trivial or well-known, but we collect them for use in the following Sections. Then, in Section \ref{sparse-vector-representation-section} we prove Theorem \ref{sparse-approximation-theorem} which gives an optimal representation of sparse vectors using deep ReLU networks and will be key to proving superconvergence in the non-linear regime $p > q$. In Section \ref{sobolev-approximation-deep-networks-section} we give the proof of the upper bounds in Theorems \ref{deep-network-upper-bound-theorem} and \ref{deep-network-upper-bound-theorem-besov}. Finally, in Section \ref{lower-bounds-section} we prove the lower bound Theorem \ref{deep-network-lower-bound-theorem} and also prove the optimality of Theorem \ref{sparse-approximation-theorem}. We remark that throughout the paper, unless otherwise specified, $C$ will represent a constant which may change from line to line, as is standard in analysis. The constant $C$ may depend upon some parameters and this dependence will be made clear in the presentation.

\section{Basic Neural Network Constructions}\label{basic-network-constructions-section}
In this section, we collect some important deep ReLU neural network constructions which will be fundamental in our construction of approximations to Sobolev and Besov functions. Many of these constructions are well-known and will be used repeatedly to construct more complex networks later on, so we collect them here for the reader's convenience.

We being by making some fundamental observations and constructing some basic networks. Much of these are trivial consequences of the definitions, but we collect them here for future reference. We begin by noting that by definition we can compose two networks by summing their depths.

\begin{lemma}[Composing Networks]\label{composition-lemma}
    Suppose $L_1,L_2\geq 1$ and that $f\in \Upsilon^{W,L_1}(\mathbb{R}^d,\mathbb{R}^k)$ and $g\in \Upsilon^{W,L_2}(\mathbb{R}^k,\mathbb{R}^l)$. Then the composition satisfies 
    \begin{equation}
        g(f(x))\in \Upsilon^{W,L_1 + L_2}(\mathbb{R}^d,\mathbb{R}^l).
    \end{equation}
    Further, if $f$ is affine, i.e. $f\in \Upsilon^{0}(\mathbb{R}^d,\mathbb{R}^k)$, then
    \begin{equation}
        g(f(x))\in \Upsilon^{W,L_2}(\mathbb{R}^d,\mathbb{R}^l).
    \end{equation}
    Finally, if instead $g$ is affine, i.e. $g\in \Upsilon^{0}(\mathbb{R}^k,\mathbb{R}^l)$ then
    \begin{equation}
        g(f(x))\in \Upsilon^{W,L_1}(\mathbb{R}^d,\mathbb{R}^l)
    \end{equation}
\end{lemma}
We remark that combining this with the simple fact that we can always increase the width of a network, we can apply Lemma \ref{composition-lemma} to networks with different widths and the width of the resulting network will be the maximum of the two widths. We will use this extension without comment in the following.

Next, we give a simple construction allowing us to apply two networks networks in parallel.
\begin{lemma}[Concatenating Networks]\label{concatenating-lemma}
    Let $d = d_1 + d_2$ and $k = k_1 + k_2$ with $d_i,k_i \geq 1$. Suppose that $f_1\in \Upsilon^{W_1,L}(\mathbb{R}^{d_1},\mathbb{R}^{k_1})$ and $f_2\in \Upsilon^{W_2,L}(\mathbb{R}^{d_2},\mathbb{R}^{k_2})$. We view $\mathbb{R}^d = \mathbb{R}^{d_1}\oplus \mathbb{R}^{d_2}$ and $\mathbb{R}^k = \mathbb{R}^{k_1}\oplus \mathbb{R}^{k_2}$. Then the function $f = f_1\oplus f_2:\mathbb{R}^d\rightarrow \mathbb{R}^k$ defined by
    \begin{equation}
        (f_1\oplus f_2)(x_1\oplus x_2) = f_1(x_1)\oplus f_2(x_2)
    \end{equation}
    satisfies $f_1\oplus f_2\in \Upsilon^{W_1 + W_2,L}(\mathbb{R}^{d},\mathbb{R}^{k})$.
\end{lemma}
\begin{proof}
    This follows by setting the weight matrices $\mathbf{W}_i = \mathbf{W}_i^1\oplus \mathbf{W}_i^2$ and $b_i = b^1_i\oplus b^2_i$, where $\mathbf{W}^1_i,b^1_i$ and $\mathbf{W}^2_i,b^2_i$ represent the parameters defining $f_1$ and $f_2$ respectively. Recally that the direct sum of matrices is simply given by
    \begin{equation}
        A\oplus B = \begin{pmatrix}
            A & \mathbf{0}\\
            \mathbf{0} & B
        \end{pmatrix}.
    \end{equation}
\end{proof}
Note that this result can be applied recursively to concatenate multiple networks. Combining this with the trivial fact that the identity map is in $\Upsilon^{2,1}(\mathbb{R},\mathbb{R})$ we see that a network can be applied to only a few components of its input.

\begin{lemma}\label{select-coordinates-lemma}
    Let $m \geq 0$ and suppose that $f\in \Upsilon^{W,L}(\mathbb{R}^{d},\mathbb{R}^k)$. Then the function $f\oplus I$ on $\mathbb{R}^{d + m}$ defined by
    \begin{equation}
        (f\oplus I)(x_1 \oplus x_2) = f(x_1)\oplus x_2
    \end{equation}
    satisfies $f\oplus I\in \Upsilon^{W+2m,L}(\mathbb{R}^{d + m}, \mathbb{R}^{k + m})$.
\end{lemma}

Using these basic lemmas we obtain the well-known construction of a deep network which represents the sum of a collection of smaller networks.
\begin{proposition}[Summing Networks]\label{summing-networks}
	Let $f_i\in \Upsilon^{W,L_i}(\mathbb{R}^d,\mathbb{R}^k)$ for $i=i,...,n$. Then we have
	\begin{equation}
		\sum_{i=1}^n f_i\in \Upsilon^{W+2d+2k,L}(\mathbb{R}^d,\mathbb{R}^k),
	\end{equation}
	where $L = \sum_{i=1}^n L_i$.
\end{proposition}
For completeness we give a detailed proof in Appendix \ref{elementary-appendix}. An important application of this is the following well-known result showing how piecewise linear continuous functions can be represented using deep networks.
\begin{proposition}\label{piecewise-linear-function-proposition}
    Suppose that $f:\mathbb{R}\rightarrow \mathbb{R}$ is a continuous piecewise linear function with $k$ pieces. Then $f\in \Upsilon^{5,k-1}(\mathbb{R})$.
\end{proposition}
For the readers convenience, we give the proof in Appendix \ref{elementary-appendix}.

Next we describe how to approximate products using deep ReLU networks. This will be necessary in the following to approximate piecewise polynomial functions. The method for doing this is based upon a construction of Telgarsky \cite{telgarsky2016benefits} and was first applied to approximating smooth functions using neural networks by Yarotsky \cite{yarotsky2017error}. This construction has since become an important tool in the analysis of deep ReLU networks and has been used by many different authors \cite{devore2021neural,lu2021deep,petersen2018optimal}. For the readers convenience, we reproduce a complete description of the construction in Appendix \ref{product-network-appendix}.
\begin{proposition}[Product Network, Proposition 3 in \cite{yarotsky2017error}]\label{produt-network-proposition}
    Let $k \geq 1$. Then there exists a network $f_k\in \Upsilon^{13,6k+3}(\mathbb{R}^2)$ such that for all $x,y\in [-1,1]$ we have
    \begin{equation}
        |f_k(x,y) - xy| \leq 6\cdot 4^{-k}.
    \end{equation}
\end{proposition}

The key to obtaining superconvergence for deep ReLU networks is the bit extraction technique, which was first introduced in \cite{bartlett1998almost} with the goal of lower bounding the VC dimension of the class of neural networks with polynomial activation function. This technique as also been used to obtain sharp approximation results for deep ReLU networks \cite{yarotsky2018optimal,shen2022optimal}. In the following Proposition, which is a minor modification of Lemma 11 in \cite{bartlett2019nearly}, we construct the bit extraction networks that we will need in our approximation of Sobolev and Besov functions. For the readers convenience, we give the complete proof in Appendix \ref{bit-extraction-appendix}.
\begin{proposition}[Bit Extraction Network]\label{bit-extractor-network-proposition}
    Let $n \geq m \geq 0$ be an integer. Then there exists a network $f_{n,m}\in \Upsilon^{9,4m}(\mathbb{R},\mathbb{R}^2)$ such that for any input $x\in [0,1]$ with at most $n$ non-zero bits, i.e.
    \begin{equation}\label{binary-expansion-x}
        x = 0.x_1x_2\cdots x_n
    \end{equation}
    with bits $x_i\in \{0,1\}$, we have
    \begin{equation}
        f_{n,m}(x) = \begin{pmatrix}
            0.x_{m+1}\cdots x_n\\
            x_1x_2\cdots x_m.0
        \end{pmatrix}.
    \end{equation}
\end{proposition}

Finally, in order to deal with the case when the error is measured in $L_\infty$, we will need the following technical construction. We construct a ReLU network which takes an input in $\mathbb{R}^d$ and returns the $k$-th largest entry.
The first step is the following simple Lemma, whose proof can be found in Appendix \ref{elementary-appendix}.
\begin{lemma}[Max-Min Networks]\label{max-min-networks-lemma}
	There exists a network $p\in \Upsilon^{4,1}(\mathbb{R}^2,\mathbb{R}^2)$ such that
	\begin{equation}
		p\left(\begin{pmatrix}
            x\\
            y
        \end{pmatrix}\right) = \begin{pmatrix}
            \max(x,y)\\
            \min(x,y)
        \end{pmatrix}.
	\end{equation}
\end{lemma}
Using these networks as building blocks, we can implement a sorting network using deep ReLU neural networks.
\begin{proposition}
	Let $k \geq 1$ and $d = 2^k$ be a power of $2$. Then there exists a network $s\in \Upsilon^{4d,L}(\mathbb{R}^d,\mathbb{R}^d)$ where $L = \binom{k+1}{2}$ which sorts the input components.
\end{proposition}
Note that the power of $2$ assumption is for simplicity and is not really necessary. It is also known that the depth $\binom{k+1}{2}$ can be replaced by a multiple $Ck$ where $C$ is a very large constant \cite{ajtai19830,paterson1990improved}, but this will not be important in our argument.
\begin{proof}
	Suppose that $(i_1,j_1),...,(i_{2^{k-1}},j_{2^{k-1}})$ is a pairing of the indices of $\mathbb{R}^d$. By Lemma \ref{max-min-networks-lemma} and Lemma \ref{concatenating-lemma}, there exists a network $g\in \Upsilon^{4d,1}(\mathbb{R}^d,\mathbb{R}^d)$ which satisfies for all $l=1,...,k-1$	
	\begin{equation}
	g(x)_{i_l} = \max(x_{i_l},x_{j_l}),~g(x)_{j_l} = \min(x_{i_l},x_{j_l}),
	\end{equation}
	i.e. which sorts the entries in each pair. By a well-known construction of sorting networks (for instance bitonic sort \cite{batcher1968sorting}), composing $\binom{k+1}{2}$ such functions can be used to sort the input.
\end{proof}
Finally, we note that by selecting a single output (which is an affine map), we can obtain a network which outputs any order statistic.
\begin{corollary}\label{order-statistic-network-corollary}
	Let $1\leq \tau\leq d$ and $d = 2^k$ is a power of $2$. Then there exists a network $g_\tau\in \Upsilon^{4d,L}(\mathbb{R}^d)$ with $L = \binom{k+1}{2}$ such that
	\begin{equation}
		g_\tau(x) = x_{(\tau)},
	\end{equation}
	where $x_{(\tau)}$ is the $\tau$-th largest entry of $x$.
\end{corollary}

\section{Optimal Representation of Sparse Vectors using Deep ReLU Networks}\label{sparse-vector-representation-section}
In this section, we prove the main technical result which enables the efficient approximation of Sobolev and Besov functions in the non-linear regime when $q < p$. Specifically, we have the following Theorem showing how to optimally represent sparse integer vectors using deep ReLU neural networks.
\begin{theorem}\label{sparse-approximation-theorem}
    Let $M \geq 1$ and $N \geq 1$ and $\normalfont{\textbf{x}}\in \mathbb{Z}^N$ be an $N$-dimensional vector satisfying
    \begin{equation}\label{l-1-bound-a-eq}
        \|\normalfont{\textbf{x}}\|_{\ell^1} \leq M.
    \end{equation}
    Then if $N\geq M$, there exists a neural network $g\in \Upsilon^{17,L}(\mathbb{R},\mathbb{R})$ with depth $L\leq C\sqrt{M(1+\log(N/M))}$ which satisfies $g(n) = \normalfont{\textbf{x}}_n$ for $n=1,...,N$.
    
    Further, if $N < M$, then there exists a neural network $g\in \Upsilon^{17,L}(\mathbb{R},\mathbb{R})$ with depth $L\leq C\sqrt{N(1+\log(M/N))}$ which satisfies $g(n) = \normalfont{\textbf{x}}_n$ for $n=1,...,N$.
\end{theorem}
Before proving this theorem, we explain the meaning of the result and give some intuition. We let
\begin{equation}\label{definition-of-S-N-M}
    S_{N,M} = \{\textbf{x}\in \mathbb{Z}^N,~\|\normalfont{\textbf{x}}\|_{\ell^1} \leq M\}
 \end{equation}
 denote the set of integer vectors which we wish to encode. We can estimate the cardinality of this set as follows. Using a stars and bars argument we see that
 $$
 |\{\textbf{x}\in \mathbb{Z}_{\geq 0}^N,~\|\normalfont{\textbf{x}}\|_{\ell^1} \leq M\}| = \binom{N+M}{N} = \binom{N+M}{M}.
 $$
 Further, the signs of each non-zero entry of the above set can be chosen arbitrarily. The number of such choices of sign is equal to the number of non-zero entries and is at most $\min\{M,N\}$. This gives the bound
 \begin{equation}
     |S_{N,M}| \leq \begin{cases}
        2^M\binom{N+M}{M} & N \geq M\\
        2^N\binom{N+M}{N} & N < M.
     \end{cases}
 \end{equation}
 Taking logarithms and utilizing the bound from Lemma \ref{binomial-sum-bound} (proved later), we estimate
 \begin{equation}\label{S-N-M-size-bound-equation}
    \log |S_{N,M}| \leq C\begin{cases}
        M(1+\log(N/M)) & N \geq M\\
        N(1+\log(M/N)) & N < M,
     \end{cases}
 \end{equation}
 and this controls the number of bits required to encode the set $S_{N,M}$. Theorem \ref{sparse-approximation-theorem} implies that using deep ReLU neural networks, the number of parameters required is the \textit{square root} of the number of bits required for such an encoding. This is analogous to the original application of bit extraction \cite{bartlett1998almost} and underlies the superconvergence phenomenon.
Finally, we note that in Theorem \ref{sparse-approximation-lower-bound-theorem} from Section \ref{lower-bounds-section} we prove that Theorem \ref{sparse-approximation-theorem} itself is optimal as long as $M$ is not exponentially small or exponentially large relative to $N$.

\begin{proof}[Proof of Theorem \ref{sparse-approximation-theorem}]
 Let $M\geq 1$ and $N\geq 1$ be fixed. There are two cases to consider, when $N \geq M$ and when $N < M$.  The key to the construction in both cases will be an explicit length $k$ binary encoding of the set $S_{N,M}$ defined in equation \eqref{definition-of-S-N-M}.
 
 By a length $k$ binary encoding we mean a pair of maps:
 \begin{itemize}
    \item $E:S_{N,M}\rightarrow \{0,1\}^{\leq k}$ (an encoding map which maps $S_{N,M}$ to a bit-string of length at most $k$)
    \item $D:\{0,1\}^{\leq k}\rightarrow S_{N,M}$ (a decoding map which recovers $\textbf{x} \in S_{N,M}$ from a bit-string of length at most $k$)
 \end{itemize}
 which satisfy
 \begin{equation}
    D(E(\textbf{x})) = \textbf{x}.
 \end{equation}
 Note that the bound in equation \ref{S-N-M-size-bound-equation} implies that there exists such an encoding as long as
 $$
    k \geq C\begin{cases}
        M(1+\log(N/M)) & N \geq M\\
        N(1+\log(M/N)) & N < M.
     \end{cases}
 $$
 However, in order to construct deep ReLU networks which prove Theorem \ref{sparse-approximation-theorem}, we will need to construct encoding and decoding maps $E$ and $D$ which are given by an \textit{explicit, simple algorithm}. These will then be used to construct the neural network $g$.

 Let us begin with the first case, when $N \geq M$. In this case, we set $k = 2M(3 + \lceil\log(N/M)\rceil)$ (note that all logarithms are taken with base $2$). The encoding map $E$ is defined as
 \begin{equation}
    E(\textbf{x}) = f_1t_1f_2t_2\cdots f_{R}t_{R},
 \end{equation}
 the concatenation of $R\leq 2M$ blocks consisting of $f_i\in \{0,1\}^{1 + \lceil\log(N/M)\rceil}$ and $t_i\in \{0,1\}^2$. The $f_i$-bits encode an offset in $\{0,1,...,\lceil N/M\rceil\}$ (via binary expansion), and the $t_i$-bits encode a value in $\{0,\pm 1\}$ (via $0 = 00$, $1 = 10$, and $-1 = 01$). The $f_i$ and $t_i$ are determined from the input $\textbf{x}\in S_{N,M}$ by Algorithm \ref{encoding-small-l1-norm}.
 
 It is clear that the number of blocks $R$ produced by Algorithm \ref{encoding-small-l1-norm} is at most $2M$ since in each round of the while loop either $f_i = \lceil N/M\rceil$ (which can happen at most $M$ times before the index $j$ reaches the end of the vector) or the entry $\textbf{r}_j$ is decremented (which can happen at most $M$ times since $\|\textbf{x}\|_{\ell^1}\leq M$).
 \begin{algorithm}
\caption{Small $\ell^1$-norm Encoding Algorithm}\label{encoding-small-l1-norm}
\begin{algorithmic}[1]
\REQUIRE
$\textbf{x}\in \mathbb{Z}^N,~\|x\|_{\ell^1}\leq M$
\State Set $j=0$, $\textbf{r} = \textbf{x}$ \COMMENT{Set pointer right before the beginning of the input $\textbf{x}$ and the residual to $\textbf{x}$}
\State Set $i=1$
\WHILE{$\textbf{r}\neq 0$}
\State $l = \min\{i:~\textbf{r}_i \neq 0\}$
\COMMENT{Find the first non-zero index in the residual}
\IF {$l-j \leq \lceil N/M\rceil$} \COMMENT{If we can make it to the next non-zero index, do so}
\State $f_i = l-j$
\State $j=l$
\ELSE 
~\COMMENT{Otherwise go as far as we can}
\State $f_i = \lceil N/M\rceil$
\State $j = j + \lceil N/M\rceil$
\ENDIF
\IF {$j=l$} 
\COMMENT{If we are at the next non-zero index, $t_i$ captures its sign}
\State $t_i = \sign(\textbf{r}_j)$
\State $\textbf{r}_j = \textbf{r}_j - t_i$
\COMMENT{This decrements $\|\textbf{r}\|_{\ell^1}$ which can happen at most $M$ times}
\ELSE 
~\COMMENT{This can only happen if $f_i = \lceil N/M\rceil$, which can occur at most $M$ times}
\State $t_i = 0$
\ENDIF
\State $i = i+1$
\ENDWHILE
\end{algorithmic}
\end{algorithm}

Next, we consider the case $N < M$. In this case we set $k = 2N(3 + \lceil\log(M/N)\rceil)$, and define the encoding map $E$ via
 \begin{equation}
    E(\textbf{x}) = t_1f_1t_2f_2\cdots t_{R}f_{R},
 \end{equation}
 i.e. $E(\textbf{x})$ is the concatenation of $R\leq 2N$ blocks consisting of $t_i\in \{0,1\}^{2 + \lceil\log(M/N)\rceil}$ and $f_i\in \{0,1\}$. The $f_i$-bits encode an offset in $\{0,1\}$, and the $t_i$-bits encode a value in $\{-\lceil M/N\rceil,...,\lceil M/N\rceil\}$. Here the first bit of each $t_i$ determines its sign, while the remaining $1+\lceil\log(M/N)\rceil$ bits consist of the binary expansion of its magnitude (which lies in $\{0,...,\lceil M/N\rceil\}$). The $t_i$ and $f_i$ are determined from the input $\textbf{x}\in S_{N,M}$ by Algorithm \ref{encoding-large-l1-norm}.

 It is clear that the number of blocks $R$ produced by Algorithm \ref{encoding-large-l1-norm} is at most $2N$ since in each round of the while loop either the entry $\textbf{r}_j$ is decremented by at least $\lceil M/N\rceil$ (which can happen at most $N$ times since $\|\textbf{x}\|_{\ell^1}\leq M$), or the entry $\textbf{r}_j$ is zeroed out (which can happen at most $N$ times before $\textbf{r} = 0$ since there are only $N$ entries).

 \begin{algorithm}
\caption{Large $\ell^1$-norm Encoding Algorithm}\label{encoding-large-l1-norm}
\begin{algorithmic}[1]
\REQUIRE
$\textbf{x}\in \mathbb{Z}^N,~\|x\|_{\ell^1}\leq M$
\State Set $j=0$, $\textbf{r} = \textbf{x}$ \COMMENT{Set pointer right before the beginning of the input $\textbf{x}$ and the residual to $\textbf{x}$}
\State Set $i=1$
\WHILE{$\textbf{r}\neq 0$}
\IF {$j=0$ or $\textbf{r}_j = 0$} \COMMENT{If the value at the current index is $0$, then shift the index}
    \State $f_i = 1$ 
    \State $j = j+1$
\ELSE
    \State $f_i = 0$
\ENDIF
\IF {$|\textbf{r}_j| \leq \lceil M/N\rceil$} \COMMENT{If we can fully capture the current value, do so}
\State $t_i = \textbf{r}_j$ 
\State $\textbf{r}_j = 0$ \COMMENT{This zeros out an entry, which can happen at most $N$ times}
\ELSE 
~\COMMENT{Otherwise capture as much as we can}
\State $t_i = \sign(\textbf{r}_j)\lceil M/N\rceil$
\State $\textbf{r}_j = \textbf{r}_j - t_i$
\COMMENT{This reduces $\|\textbf{r}\|_{\ell^1}$ by at least $\lceil M/N\rceil$ which can happen at most $N$ times}
\ENDIF
\State $i = i+1$
\ENDWHILE
\end{algorithmic}
\end{algorithm}

In both cases, the decoding map $D$ is given by algorithm \ref{decoding-algorithm}. It is easy to verify that algorithm \ref{decoding-algorithm} reconstructs the input $\textbf{x}$ from the output of either algorithm \ref{encoding-small-l1-norm} or \ref{encoding-large-l1-norm}.

\begin{algorithm}
\caption{Decoding Algorithm}\label{decoding-algorithm}
\begin{algorithmic}[1]
\REQUIRE
A bit string $f_1t_1\cdots f_{R}t_{R}$
\State Set $\textbf{x}=0$ and $j=0$ \COMMENT{Start with the $0$ vector}
\FOR{$i=1,...,R$}
\State $j = j + f_i$ \COMMENT{Shift index by $f_i$}
\State $\textbf{x}_j = \textbf{x}_j + t_i$ \COMMENT{Increment value by $t_i$}
\ENDFOR
\end{algorithmic}
\end{algorithm}

We now show how to use these algorithms to construct an appropriate deep ReLU neural network $g$. Let $S$ be a threshold parameter, to be chosen later.

Given a vector $\textbf{x}\in \mathbb{Z}^N$, we decompose it into two pieces $\textbf{x} = \textbf{x}^B + \textbf{x}^s$ (here $\textbf{x}^B$ represents the `big' part and $\textbf{x}^s$ the `small' part). We define
\begin{equation}
	\textbf{x}^B_i = \begin{cases}
		|\textbf{x}_i| & \textbf{x}_i \geq S\\
		0 & \textbf{x}_i < S
	\end{cases}
\end{equation}
and
\begin{equation}
	\textbf{x}^s_i = \begin{cases}
		0 & \textbf{x}_i \geq S\\
		|\textbf{x}_i| & \textbf{x}_i < S
	\end{cases}
\end{equation}
The large part $\textbf{x}^B$ has small support and so can be efficiently encoded as a piecewise linear function. Specifically, the $\ell^1$-norm bound \eqref{l-1-bound-a-eq} on $\textbf{x}$ implies that the support of $\textbf{x}^B$ is at most of size
\begin{equation}
	|\{n:~\textbf{x}^B_n \neq 0\}| \leq \frac{\|\textbf{x}^B\|_{\ell^1}}{S} \leq \frac{\|\textbf{x}\|_{\ell^1}}{S} \leq \frac{M}{S}.
\end{equation}
This means that there is a piecewise linear function with at most $3M/S$ pieces which matches the values of $\textbf{x}^B$, so by Proposition \ref{piecewise-linear-function-proposition} there exists a network $$g_B\in \Upsilon^{5,L}(\mathbb{R})$$ 
with depth bounded by $L \leq 3M/S$ such that $g_B(n) = \textbf{x}^B_n$ for $n=1,...,N$.

The heart of the proof is an efficient encoding of the small part $\textbf{x}^s$. This requires the encoding and decoding algorithms \ref{encoding-small-l1-norm}, \ref{encoding-large-l1-norm} and \ref{decoding-algorithm}. We consider first the case $M \leq N$, which is captured in the following Proposition.

\begin{proposition}\label{small-l-1-small-reconstrution-proposition}
    Let $M\leq N$ and suppose that $\normalfont{\textbf{x}}\in \mathbb{Z}^N$ and satisfies $\|\normalfont{\textbf{x}}\|_{\ell^l} \leq M$ and $\|\normalfont{\textbf{x}}\|_{\ell^\infty} < S$. Then there exists a $g\in \Upsilon^{15,L}(\mathbb{R})$ such that $g(n) = \normalfont{\textbf{x}}_n$ for $n=1,...,N$ with
    \begin{equation}
        L \leq 8M/S + 8S(5+\lceil \log(N/M)\rceil) + 4.
    \end{equation}
\end{proposition}
The proof is quite technical, so let us give a high-level description of the ideas first. The idea is to take the execution of the decoding algorithm \ref{decoding-algorithm} which reconstructs $\textbf{x}$ and to divide it into blocks of length on the order of $S$. Each block will start at a point $i$ in the algorithm at which $\textbf{x}_j = 0$ before step $4$ of the loop in \ref{decoding-algorithm}. During the execution of this block, the index $j$ increases and reaches a larger value at the end of the block. All of the entries $\textbf{x}_n$ for $n$ between these values is reconstructed during the given block of the reconstruction algorithm.

We now construct three networks. Given an input index $n$, find the block during which the value $\textbf{x}_n$ is reconstructed. On the input index $n$, one network outputs the value of $j$ at the beginning of the block, and another network outputs a real number whose binary expansion contains the bits consumed during this block. Both of these can be implemented using piecewise linear functions whose number of pieces is proportional to the number of blocks. The final network extracts the bits from the output of the second network and implements the execution of algorithm \ref{decoding-algorithm} in this block to reconstruct the value of $\textbf{x}$.

Before giving the detailed proof of Proposition \ref{small-l-1-small-reconstrution-proposition}, let us complete the proof of Theorem \ref{sparse-approximation-theorem} in the case $M\leq N$. We apply Proposition \ref{small-l-1-small-reconstrution-proposition} to the small part $\textbf{x}^s$ to get a network $g^s$. Then we use Proposition \ref{summing-networks} to add this network to the network $g^B$ representing the large part $\textbf{x}^B$ to get a network $g\in \Upsilon^{17,L}(\mathbb{R})$ with
\begin{equation}
    L \leq 11M/S + 8S(5+\lceil \log(N/M)\rceil) + 4
\end{equation}
such that $g(n) = \textbf{x}_n$ for $n=1,...,N$. Finally, we choose $S$ optimally, namely
$$
    S = \sqrt{\frac{M}{5+\lceil \log(N/M)\rceil}},
$$
to get $L\leq C\sqrt{M(1+\log(N/M))}$ as desired.

\begin{proof}[Proof of Proposition \ref{small-l-1-small-reconstrution-proposition}]
    Let $f_1t_1\cdots f_{R}t_{R}$ be the output of the encoding Algorithm \ref{encoding-small-l1-norm} run on input $\textbf{x}^s$. Let
    \begin{equation}
        F := \{i\in \{1,...,R\}:~f_i = 0\}.
    \end{equation}
    We decompose the set $F$ into intervals, i.e.
    \begin{equation}
        F = \bigcup_{m=1}^{\mathcal{T}} [B_m, U_m],
    \end{equation}
    where $[I,J] := \{I,I+1,...,J\}$ for $I \leq J$ and $B_{m+1} > U_m+1$.
    
    Note that since $\|\textbf{x}\|_{\ell^\infty} < S$, the length of these intervals is strictly less than $S$, i.e. $U_m - B_m + 1 < S$ for $m=1,...,\mathcal{T}$. This holds since the encoding algorithm \ref{encoding-small-l1-norm} stays at the same index for all steps $i\in [B_m, U_m]$. Hence this index is decremented $U_m - B_m + 1$ times and so this quantity must be smaller than $S$.

    Let $\rho = \lceil R/S\rceil$ and consider steps $i_0,...,i_\rho$ defined by $i_{\rho} = R+1$ (the end of the algorithm) and 
    \begin{equation}
        i_k = \begin{cases}
            1+kS & 1+kS\notin F\\
            B_m - 1 & 1+kS\in [B_m,U_m], 
        \end{cases}
    \end{equation}
    for $k=0,...,\rho - 1$. (Note that $f_1\neq 0$ since the index $j$ starts at $0$ in algorithm \ref{encoding-small-l1-norm}. Thus $1\notin F$ and so $i_0 = 1$.)

    The bound $U_m - B_m + 1 < S$ on the length of the interval $[B_m,U_m]$ implies that $i_k > 1-(k-1)S$. This implies that $i_{k-1} < i_k$ and also that the gaps satisfy $i_k - i_{k-1} < 2S$ for all $k=1,...,\rho$.

    Next, let indices $j_k$ for $k=0,...,\rho-1$ be the values of the index $j$ at the beginning of step $i_k$ in the decoding algorithm \ref{decoding-algorithm}. We also set $j_\rho = N$. Since by construction the intervals are not consecutive, i.e. $B_{m+1} > U_m+1$, the steps $i_k\notin F$, i.e. $f_{i_k} > 0$. This means that $j_{k-1} < j_k$ for all $k=1,...,\rho$.

    Observe that the steps $i_k$ and indices $j_k$ have been constructed such that for an integer $n$ in the interval $j_k < n \leq j_{k+1}$, the value $\textbf{x}_n$ is only affected during the steps $i_k,...,i_{k+1} - 1$ in the reconstruction algorithm \ref{decoding-algorithm}. Further, the length of each block satisfies $i_k - i_{k-1} < 2S$.
    
    Next, we construct two piecewise linear functions $J$ and $R$ as follows. For integers $n=1,...,N$, we set
    \begin{equation}
        J(n) = 
            j_k~~\text{for}~~j_k < n\leq j_{k+1},
    \end{equation}
    and
    \begin{equation}
        R(n) = 
            r_k ~~\text{for}~~ j_k < n\leq j_{k+1},
    \end{equation}
    where
    \begin{equation}
        r_k = 0.f_{i_k}t_{i_k}\cdots f_{i_{k+1} - 1}t_{i_{k+1} - 1}
    \end{equation}
    is the real number whose binary expansion contains the encoding of $\textbf{x}$ from step $i_k$ to $i_{k+1} - 1$ (followed by zeros). Both $J$ and $R$ take at most $\rho+1$ different values and hence can be implemented by piecewise linear functions with at most $2\rho+1$ pieces. Thus, by Proposition \ref{piecewise-linear-function-proposition} we have $J,R\in \Upsilon^{5,2\rho}(\mathbb{R})$.

    We being our network construction as follows. We begin with the affine map
\begin{equation}
	x\rightarrow \begin{pmatrix}
            x\\
            x\\
            x
            \end{pmatrix}\in \Upsilon^0(\mathbb{R},\mathbb{R}^3),
\end{equation}
and use Lemmas \ref{composition-lemma} and \ref{select-coordinates-lemma} to apply $J$ to the first component and then apply $R$ to the second component to get
\begin{equation}
x\rightarrow \begin{pmatrix}
            J(x)\\
            x\\
            x
            \end{pmatrix}\rightarrow \begin{pmatrix}
            J(x)\\
            R(x)\\
            x
            \end{pmatrix}\in \Upsilon^{9, 4\rho}(\mathbb{R},\mathbb{R}^3).
\end{equation}
Composing with the affine map
\begin{equation}
	\begin{pmatrix}
            x\\
            y\\
            z
            \end{pmatrix}\rightarrow \begin{pmatrix}
            z-x\\
            y\\
            0
            \end{pmatrix}\in \Upsilon^0(\mathbb{R}^3,\mathbb{R}^3),
\end{equation}
and using Lemma \ref{composition-lemma} again we get that
\begin{equation}\label{first-decoding-1478}
	x\rightarrow \begin{pmatrix}
            x - J(x)\\
            R(x)\\
            0
            \end{pmatrix}\in\Upsilon^{9, 4\rho}(\mathbb{R},\mathbb{R}^3).
\end{equation}
    Applied to an integer $j_k < n\leq j_{k+1}$, this network maps
    \begin{equation}
        n\rightarrow \begin{pmatrix}
            n - j_k\\
            0.f_{i_k}t_{i_k}\cdots f_{i_{k+1} - 1}t_{i_{k+1} - 1}\\
            0
            \end{pmatrix}.
    \end{equation}
    Thus the first entry is the gap between $n$ and the index $j$ at the beginning of step $i_k$ and the last entry is the value of $\textbf{x}_n$ at the beginning of step $i_k$, while the middle entry contains the bits used by the algorithm between steps $i_k,...,i_{k+1} - 1$. The proof will now be completed by constructing a network which applies a single step of the decoding algorithm \ref{decoding-algorithm} to each of these entries, this is collected in the following technical Lemma.

    \begin{lemma}\label{algorithm-step-lemma-small-l1}
        Given positive integers $\alpha$ and $\beta$ there exists a network $g\in \Upsilon^{15,4\alpha+16}(\mathbb{R}^3,\mathbb{R}^3)$ such that
        \begin{equation}
            g:\begin{pmatrix}
            x\\
            0.f_{1}t_{1}\cdots f_{k}t_{k}\\
            \Sigma
            \end{pmatrix}\rightarrow \begin{pmatrix}
            x - f_1\\
            0.f_{2}t_{2}\cdots f_{k}t_{k}\\
            \Sigma + t_1\delta(x - f_1)
            \end{pmatrix}
        \end{equation}
        whenever $x\in \mathbb{Z}$, $k\leq \beta$ and $\text{\normalfont len}(f_i) = \alpha$. Here the $f_i$ denote integers encoded via binary expansion and $\text{\normalfont len}(f_i)$ is the length of this expansion, $t_i\in \{\pm 1, 0\}$ are encoded using two bits (specifically via $0=00$, $1 = 10$ and $-1 = 01$), $\Sigma$ denotes a running sum, and $\delta$ is the integer Dirac delta defined by
        \begin{equation}
            \delta(z) = \begin{cases}
                1 & z = 0\\
                0 & z\neq 0
            \end{cases}
        \end{equation}
        for integer inputs $z$.
    \end{lemma}
Before proving this Lemma, let us complete the proof of Proposition \ref{small-l-1-small-reconstrution-proposition}. We set $\alpha = 1 + \lceil \log(N/M)\rceil$ and $\beta = 2S$, and compose the map in \eqref{first-decoding-1478} with $2S$ copies of the network given by Lemma \ref{algorithm-step-lemma-small-l1}. Then we finally compose with an affine map which selects the last coordinate. This gives a $g\in \Upsilon^{15,L}(\mathbb{R})$ with
\begin{equation}
\begin{split}
    L = 4\rho + 2S(4\alpha + 16) &= 4\lceil R/S\rceil + 8S(5+\lceil \log(N/M)\rceil)\\
    &\leq 8M/S + 8S(5+\lceil \log(N/M)\rceil) + 4,
\end{split}
\end{equation}
since $R \leq 2M$. 

When applied to an integer $n\in \{1,...,N\}$ with $j_k < n\leq j_{k+1}$, the map in \eqref{first-decoding-1478} sets the offset between $n$ and the index $j_k$ and the start of step $i_k$, outputs a number whose binary expansion contains the bits used from step $i_k$ to step $i_{k+1}-1$, and sets a running sum to $0$. 

Then the $2S$ copies of the network from Lemma \ref{algorithm-step-lemma-small-l1} implement algorithm \ref{decoding-algorithm} from step $i_k$ to step $i_{k+1}-1$. Note that if the number of steps is less than $2S$, the network pads with zero blocks $(f_i,t_i) = 0$, and these additional steps have no effect. Since by construction the entry $\textbf{x}_n$ is only modified during these steps, the running sum will now be equal to $\textbf{x}_n$. Finally, we select the last coordinate, which guarantees that $g(n) = \textbf{x}_n$.
    \begin{proof}[Proof of Lemma \ref{algorithm-step-lemma-small-l1}]
        We construct the desired network as follows. We use Lemma \ref{select-coordinates-lemma} to apply the bit extractor network $f_{n,\alpha}$ from Proposition \ref{bit-extractor-network-proposition} to the second component. Here we choose $n \geq \beta(\alpha+2)$, which is guaranteed to be larger than the length of the bit-string in the second component. This results in the map
\begin{equation}
	\begin{pmatrix}
            x\\
            0.f_1t_1\cdots f_{k}t_{k}\\
            \Sigma
            \end{pmatrix}\rightarrow\begin{pmatrix}
            x\\
            f_1\\
            0.t_1f_2t_2\cdots f_{k}t_{k}\\
            \Sigma
            \end{pmatrix}\in \Upsilon^{13,4\alpha}(\mathbb{R}^3,\mathbb{R}^4).
\end{equation}
Subtracting the second component from the first, this gives
\begin{equation}\label{first-part-small-l1-algorithm-step-lemma}
    	\begin{pmatrix}
            x\\
            0.f_1t_1\cdots f_{k}t_{k}\\
            \Sigma
            \end{pmatrix}\rightarrow\begin{pmatrix}
            x - f_1\\
            0.t_1f_2t_2\cdots f_{k}t_{k}\\
            \Sigma
            \end{pmatrix}\in \Upsilon^{13,4\alpha}(\mathbb{R}^3,\mathbb{R}^3),
\end{equation}
and completes the first part of the construction.

Next, we implement a network which extracts the two bits corresponding to $t$ and then adds $t$ to the third component iff the first component is $0$. Let $h(z)$ denote the continuous piecewise linear function
\begin{equation}\label{definition-of-h}
	h(z) = \begin{cases}
		0 & z\leq -1\\
		z+1 & -1 < z \leq 0\\
		1-z & 0 < z \leq 1\\
		0 & z > 1.
	\end{cases}
\end{equation}
For integer inputs, $h$ is simply the delta function, i.e. $h(z) = \delta(z)$ for $z\in \mathbb{Z}$, and by Proposition \ref{piecewise-linear-function-proposition} we have $h\in \Upsilon^{5,3}(\mathbb{R})$. We first apply an affine map which duplicates the first coordinate
\begin{equation}
	\begin{pmatrix}
            z_1\\
            z_2\\
            z_3
            \end{pmatrix}\rightarrow \begin{pmatrix}
            z_1\\
            z_1\\
            z_2\\
            z_3
            \end{pmatrix}\in \Upsilon^0(\mathbb{R}^3,\mathbb{R}^4).
\end{equation}
Then, we use Lemma \ref{select-coordinates-lemma} to apply $h$ to the second coordinate and apply the bit extractor network $f_{n,1}$ from Proposition \ref{bit-extractor-network-proposition} to the third component. As before, we choose $n \geq \beta(\alpha+2)$ which is guaranteed to be larger than the length of the bit-string in the second component. This gives (note that we write $b_1b_2$ for the two bits corresponding to $t_1$)
\begin{equation}
	\begin{pmatrix}
            z_1\\
            0.b_1b_2f_2t_2...f_{k}t_{k}\\
            z_3
            \end{pmatrix}\rightarrow\begin{pmatrix}
            z_1\\
            h(z_1)\\
            b_1\\
            0.b_2f_2t_2...f_{k}t_{k}\\
            z_3
            \end{pmatrix}\in \Upsilon^{15,7}(\mathbb{R}^3,\mathbb{R}^5).
\end{equation}
Now we compose this using Lemma \ref{composition-lemma} with the map
\begin{equation}
	\begin{pmatrix}
            z_1\\
            z_2\\
            z_3\\
            z_4\\
            z_5
            \end{pmatrix}\rightarrow\begin{pmatrix}
            z_1\\
            z_2 + z_3-1\\
            z_4\\
            z_5
            \end{pmatrix}\rightarrow\begin{pmatrix}
            z_1\\
            \sigma(z_2+z_3-1)\\
            z_4\\
            z_5
            \end{pmatrix}\rightarrow\begin{pmatrix}
            z_1\\
            z_4\\
            z_5 + \sigma(z_2+z_3-1)
            \end{pmatrix}\in\Upsilon^{7,1}(\mathbb{R}^5,\mathbb{R}^3).
\end{equation}
Here the first and last maps in the composition are affine and the middle map is in $\Upsilon^{7,1}(\mathbb{R}^4,\mathbb{R}^4)$ by Lemma \ref{select-coordinates-lemma}. This gives
\begin{equation}\label{positive-bit}
		\begin{pmatrix}
            z_1\\
            0.b_1b_2f_2t_2...f_{k}t_{k}\\
            z_3
            \end{pmatrix}\rightarrow\begin{pmatrix}
            z_1\\
            0.b_2f_2t_2...f_{k}t_{k}\\
            z_3 + \sigma(h(z_1) + b_1-1)\\
            \end{pmatrix}\in \Upsilon^{15,8}(\mathbb{R}^3,\mathbb{R}^3).
\end{equation}
Notice that $\sigma(h(z_1) + b_1-1)$ equals $1$ precisely when $z_1 = 0$ and $b_1 = 1$ and equals zero otherwise (for integral $z_1$). 

In an analogous manner, we get
\begin{equation}\label{negative-bit}
    \begin{pmatrix}
            z_1\\
            0.b_2f_2t_2...f_{k}t_{k}\\
            z_3
            \end{pmatrix}\rightarrow\begin{pmatrix}
            z_1\\
            0.f_2t_2...f_{k}t_{k}\\
            z_3 - \sigma(h(z_1) + b_2-1)\\
            \end{pmatrix}\in \Upsilon^{15,8}(\mathbb{R}^3,\mathbb{R}^3).
\end{equation}
Composing the networks in \eqref{positive-bit} and \eqref{negative-bit} will extract $t_1\in \{0,\pm 1\}$ (recall the encoding $0=00$, $1 = 10$ and $-1 = 01$) and add $t_1$ to the last coordinate iff the first coordinate is $0$. Composing this with the network in \eqref{first-part-small-l1-algorithm-step-lemma} gives a network $g\in \Upsilon^{15,4\alpha + 16}(\mathbb{R}^3,\mathbb{R}^3)$ as stated in the Lemma.
    \end{proof}

\end{proof}

Next, we consider the case $M > N$, which is somewhat complicated by the fact that the threshold parameter $S$ and the spacing of the blocks are no longer equal in this case. The key construction is contained in the following Proposition.

\begin{proposition}\label{large-l-1-small-reconstrution-proposition}
    Let $M > N$ and suppose that $\normalfont{\textbf{x}}\in \mathbb{Z}^N$ and satisfies $\|\normalfont{\textbf{x}}\|_{\ell^l} \leq M$ and $\|\normalfont{\textbf{x}}\|_{\ell^\infty} < S$. Then there exists a $g\in \Upsilon^{15,L}(\mathbb{R})$ such that $g(n) = \normalfont{\textbf{x}}_n$ for $n=1,...,N$ with
    \begin{equation}
        L \leq 8M/S + 8(SN/M+1)(4 + \lceil \log(M/N)\rceil) + 4.
    \end{equation}
\end{proposition}
Utilizing this Proposition, we complete the proof of Theorem \ref{sparse-approximation-theorem} in the case $M > N$. We apply Proposition \ref{large-l-1-small-reconstrution-proposition} to $\textbf{x}^s$ and use Proposition \ref{summing-networks} to add the network to the network representing $\textbf{x}^B$ to get a network $g\in \Upsilon^{17,L}(\mathbb{R})$ representing $\textbf{x}$ with
\begin{equation}
    L\leq 11M/S +  8(SN/M+1)(4 + \lceil \log(M/N)\rceil) + 4.
\end{equation}
Finally, we optimize in $S$, resulting in a value
\begin{equation}
    S = \frac{M\sqrt{4+\lceil \log(M/N)\rceil}}{\sqrt{N}}
\end{equation}
to get $L \leq C\sqrt{N(1+\log(M/N)}$ as desired.
\begin{proof}[Proof of Proposition \ref{large-l-1-small-reconstrution-proposition}]
    The proof proceeds in a very similar manner to the proof of Proposition \ref{small-l-1-small-reconstrution-proposition} and we only indicate the differences here.

    We begin with the same set $F$ and its decomposition into intervals $[B_m,U_m]$, except that $f_1t_1\cdots f_Rt_R$ is now the output of the encoding algorithm \ref{encoding-large-l1-norm}.

    Our bound on the block length becomes $U_m - B_m + 1 \leq SN/M$. This holds since the encoding algorithm stays at the same index for all steps $i\in [B_m,U_m]$, and thus this index in decremented by an amount $\lceil M/N\rceil$ a total of $U_m - B_m + 1$ times. The bound on the $\ell_\infty$-norm implies that $(M/N)(U_m - B_m + 1) < S$, which gives the desired bound.

    Thus, in this case we set $T = \lceil SN/M\rceil$ and $\rho = \lceil R/T\rceil$ and consider steps $i_0,...,i_\rho$ defined by $i_{\rho} = R+1$ (the end of the algorithm) and 
    \begin{equation}
        i_k = \begin{cases}
            1+kT & 1+kT\notin F\\
            B_m - 1 & 1+kT\in [B_m,U_m], 
        \end{cases}
    \end{equation}
    for $k=0,...,\rho - 1$.

    We now proceed with the same argument as in Proposition \ref{small-l-1-small-reconstrution-proposition}, except that the bound on $U_m - B_m + 1 < T$ implies that all block lengths are bounded by $2T$. The proof is finally completed with the following variant of Lemma \ref{algorithm-step-lemma-small-l1}, which implements a step of the decoding algorithm \ref{decoding-algorithm} with the values $f_i$ and $t_i$ encoded as they are for $M > N$.
    \begin{lemma}\label{large-l-1-small-reconstrution-lemma}
        Given positive integers $\alpha$ and $\beta$ there exists a network $g\in \Upsilon^{15,4\alpha + 8}(\mathbb{R}^3,\mathbb{R}^3)$ such that
        \begin{equation}
            g:\begin{pmatrix}
            x\\
            0.f_{1}t_{1}\cdots f_{k}t_{k}\\
            \Sigma
            \end{pmatrix}\rightarrow \begin{pmatrix}
            x - f_1\\
            0.f_{2}t_{2}\cdots f_{k}t_{k}\\
            \Sigma + t_1\delta(x - f_1)
            \end{pmatrix}
        \end{equation}
        whenever $x\in \mathbb{Z}$, $k\leq \beta$ and $\text{\normalfont len}(t_i) = \alpha$. Here $f_i\in \{0,1\}$ are single bits, $t_i\in \mathbb{Z}$ is encoded via binary expansion with a single bit giving its sign and $\text{\normalfont len}(t_i)$ is the length of this expansion, $\Sigma$ denotes a running sum, and $\delta$ is the integer Dirac delta defined by
        \begin{equation}
            \delta(z) = \begin{cases}
                1 & z = 0\\
                0 & z\neq 0
            \end{cases}
        \end{equation}
        for integer inputs $z$.
    \end{lemma}
    Given this lemma, we complete the proof as before, setting $\alpha = 2+\lceil \log(M/N)\rceil$ and $\beta = 2T$ and composing the network implementing the maps $J$ and $R$ with $2T$ copies of the network from Lemma \ref{large-l-1-small-reconstrution-lemma}. This gives a network $g\in \Upsilon^{15,L}(\mathbb{R})$ with
    \begin{equation}
    \begin{split}
        L \leq 4\rho + 2T(4\alpha+8) &= 4\lceil R/T\rceil + 8T(4 + \lceil \log(M/N)\rceil)\\
        &\leq 8N/T + 8T(4 + \lceil \log(M/N)\rceil) + 4\\
        &\leq 8M/S + 8(SN/M+1)(4 + \lceil \log(M/N)\rceil) + 4,
    \end{split}
    \end{equation}
    since $R\leq 2N$ and $T = \lceil SN/M\rceil$.
    \begin{proof}[Proof of Lemma \ref{large-l-1-small-reconstrution-lemma}]
        We use Lemma \ref{select-coordinates-lemma} to apply the bit extractor network $f_{n,1}$ from Proposition \ref{bit-extractor-network-proposition} to the second component. Here we choose $n \geq \beta(\alpha+1)$, which is guaranteed to be larger than the length of the bit-string in the second component. Then we subtract the second component from the first. This results in the map
\begin{equation}\label{first-part-large-l1-algorithm-step-lemma}
    	\begin{pmatrix}
            x\\
            0.f_1t_1\cdots f_{k}t_{k}\\
            \Sigma
            \end{pmatrix}\rightarrow\begin{pmatrix}
            x - f_1\\
            0.t_1f_2t_2\cdots f_{k}t_{k}\\
            \Sigma
            \end{pmatrix}\in \Upsilon^{13,4}(\mathbb{R}^3,\mathbb{R}^3),
\end{equation}
and completes the first part of the construction.

Now we wish to extract the integer $t_1$ and add it to $\Sigma$ iff the first coordinate (which is an integer) is $0$. We do this by using Lemma \ref{select-coordinates-lemma} to apply the bit extractor network $f_{n,1}$ to the second coordinate and then apply $f_{n,\alpha - 1}$ to the third coordinate of the result to get
\begin{equation}\label{bits-extracted-network-large-1910}
	\begin{pmatrix}
            z_1\\
            0.b_1b_2...b_\alpha f_2t_2...f_{k}t_{k}\\
            z_3
            \end{pmatrix}\rightarrow\begin{pmatrix}
            z_1\\
            b_1\\
            b_2...b_\alpha\\
            0.f_2t_2...f_{k}t_{k}\\
            z_3
            \end{pmatrix}\in \Upsilon^{15,4\alpha}(\mathbb{R}^3,\mathbb{R}^5),
\end{equation}
where we have written $b_1b_2...b_\alpha$ for the bits of $t_1$.

Next, consider the following sequence of compositions, where $h$ is the function defined in \eqref{definition-of-h},
\begin{equation}\label{first-network-1991}
\begin{split}
    	\begin{pmatrix}
            z_1\\
            z_2\\
            z_3\\
            z_4\\
            z_5
            \end{pmatrix}\rightarrow
            \begin{pmatrix}
            z_1\\
            z_1\\
            z_2\\
            z_3\\
            z_4\\
            z_5
            \end{pmatrix}\rightarrow
            \begin{pmatrix}
            z_1\\
            h(z_1)\\
            z_2\\
            z_3\\
            z_4\\
            z_5
            \end{pmatrix}\rightarrow
            \begin{pmatrix}
            z_1\\
            z_2\\
            z_3\\
            z_3 - 2^\alpha(1 - h(z_1) + z_2)\\
            z_4\\
            z_5
            \end{pmatrix}
            \rightarrow&\\
            \rightarrow
            \begin{pmatrix}
            z_1\\
            z_2\\
            z_3\\
            \sigma(z_3 - 2^\alpha(1 - h(z_1) + z_2))\\
            z_4\\
            z_5
            \end{pmatrix}\rightarrow&\\
            \rightarrow
            \begin{pmatrix}
            z_1\\
            z_2\\
            z_3\\
            z_4\\
            z_5 + \sigma(z_3 - 2^\alpha(1 - h(z_1) + z_2))
            \end{pmatrix}&\in \Upsilon^{15,4}(\mathbb{R}^5,\mathbb{R}^5).
\end{split}
\end{equation}
Using a sequence of applications of Lemmas \ref{select-coordinates-lemma} and \ref{composition-lemma}, we obtain that this map can be implemented by a network in $ \Upsilon^{15,4}(\mathbb{R}^5,\mathbb{R}^5)$.

Note that when $z_1\in \mathbb{Z}$ and $z_2\in \{0,1\}$, we have that (recall that $h(z) = \delta(z)$ for integer $z$)
\begin{equation}
    (1 - h(z_1) + z_2) = \begin{cases}
        0 & z_1 = 0~\text{and}~z_2 = 0\\
        1 & z_1 \neq 0~\text{and}~z_2 = 0\\
        1 & z_1 = 0~\text{and}~z_2 = 1\\
        2 & z_1 \neq 0~\text{and}~z_2 = 1.
    \end{cases}
\end{equation}
If we also have that $z_3\in \{0,...,2^\alpha\}$, then it follows that
\begin{equation}
    \sigma(z_3 - 2^\alpha(1 - h(z_1) + z_2)) = \begin{cases}
        z_3 & z_1 = 0~\text{and}~z_2 = 0\\
        0 & \text{otherwise}.
    \end{cases}
\end{equation}
Thus, if we compose the network in \eqref{bits-extracted-network-large-1910} with the network in \eqref{first-network-1991}, we will add the number $b_2...b_\alpha$ (which is less than $2^\alpha$) to the last coordinate iff $z_1 = b_1 = 0$. As $b_1 = 0$ to indicate that $t_1$ is positive and $b_2...b_\alpha$ contain the value of $t_1$, his handles the case where $t_1$ is positive and has no effect when $t_1$ is negative.

Next, we construct a network which handles the negative part of $t_1$. This is given by the following composition
\begin{equation}\label{second-network-2000}
\begin{split}
    	\begin{pmatrix}
            z_1\\
            z_2\\
            z_3\\
            z_4\\
            z_5
            \end{pmatrix}\rightarrow
            \begin{pmatrix}
            z_1\\
            z_1\\
            z_2\\
            z_3\\
            z_4\\
            z_5
            \end{pmatrix}\rightarrow
            \begin{pmatrix}
            z_1\\
            h(z_1)\\
            z_2\\
            z_3\\
            z_4\\
            z_5
            \end{pmatrix}\rightarrow
            \begin{pmatrix}
            z_1\\
            z_2\\
            z_3\\
            z_3 - 2^\alpha(2 - h(z_1) - z_2)\\
            z_4\\
            z_5
            \end{pmatrix}
            \rightarrow&\\
            \rightarrow
            \begin{pmatrix}
            z_1\\
            z_2\\
            z_3\\
            \sigma(z_3 - 2^\alpha(2 - h(z_1) - z_2))\\
            z_4\\
            z_5
            \end{pmatrix}\rightarrow&\\
            \rightarrow
            \begin{pmatrix}
            z_1\\
            z_2\\
            z_3\\
            z_4\\
            z_5 - \sigma(z_3 - 2^\alpha(2 - h(z_1) - z_2))
            \end{pmatrix}&\in \Upsilon^{15,4}(\mathbb{R}^5,\mathbb{R}^5).
\end{split}
\end{equation}
When $z_1\in \mathbb{Z}$ and $z_2\in \{0,1\}$, we have that
\begin{equation}
    (2 - h(z_1) - z_2) = \begin{cases}
        1 & z_1 = 0~\text{and}~z_2 = 0\\
        2 & z_1 \neq 0~\text{and}~z_2 = 0\\
        0 & z_1 = 0~\text{and}~z_2 = 1\\
        1 & z_1 \neq 0~\text{and}~z_2 = 1.
    \end{cases}
\end{equation}
So, if we compose the network in \eqref{bits-extracted-network-large-1910} with the network in \eqref{second-network-2000}, we will subtract the number $b_2...b_\alpha$ (which is less than $2^\alpha$) from the last coordinate iff $z_1 = 0$ and $b_1 = 1$. As $b_1 = 1$ to indicate that $t_1$ is negative and $b_2...b_\alpha$ contain the value of $t_1$, this handles the case where $t_1$ is negative and has no effect when $t_1$ is positive.

We obtain the final network $g$ by successively composing the network in \eqref{bits-extracted-network-large-1910} with the networks in \eqref{first-network-1991} and \eqref{second-network-2000} and then dropping the second and third components.
    \end{proof}
\end{proof}

\end{proof}

\section{Optimal Approximation of Sobolev Functions Using Deep ReLU Networks}\label{sobolev-approximation-deep-networks-section}
In this section, we give the main construction and the proof of Theorems \ref{deep-network-upper-bound-theorem} and \ref{deep-network-upper-bound-theorem-besov}. A key component of the proof is the approximation of piecewise polynomial functions using deep ReLU neural networks. To describe this, we first introduce some notation. 

Throughout this section, unless otherwise specified, let $b \geq 2$ be a fixed integer. To avoid excessively cumbersome notation, we suppress the dependence on $b$ in the following notation. Let $l \geq 0$ be an integer and consider the $b$-adic decomposition of the cube $\Omega = [0,1)^d$ (note that by removing a zero-measure set it suffices to consider this half-open cube in the proof) at level $l$ given by
\begin{equation}\label{multiscale-decomposition-l}
    \Omega = \bigcup_{\textbf{i}\in I_l} \Omega_{\textbf{i}}^l,
\end{equation}
where the index $\textbf{i}$ lies in the index set $I_l := \{0,...,b^l-1\}^d$, and $\Omega^l_\textbf{i}$ is defined by
\begin{equation}\label{subcubes-definition}
    \Omega^l_{\textbf{i}} = \prod_{j=1}^d [b^{-l}\textbf{i}_j,b^{-l}(\textbf{i}_j + 1)).
\end{equation}
Note that for each $l$, the $b^{dl}$ subcubes $\Omega^l_{\textbf{i}}$ form a partition of the original cube $\Omega$. For an integer $k \geq 0$, we let $P_k$ denote the space of polynomials of degree at most $k$ and consider the space
\begin{equation}
    \mathcal{P}^l_k = \left\{f:\Omega\rightarrow \mathbb{R},~f_{\Omega^l_{\textbf{i}}}\in P_k~\text{for all $\textbf{i}\in I_l$}\right\}
\end{equation}
of (non-conforming) piecewise polynomials subordinate to the partition \eqref{multiscale-decomposition-l}. The space $\mathcal{P}^l_k$ has dimension $\binom{d+k}{k}b^{dl}$ and a natural ($L_\infty$-normalized) basis
\begin{equation}
    \rho_{l,\textbf{i}}^\alpha(x) = \begin{cases}
        \prod_{j=1}^d (b^lx_j - \textbf{i}_j)^{\alpha_j} & x\in \Omega^l_{\textbf{i}}\\
        0 & x\notin \Omega^l_{\textbf{i}}
    \end{cases}
\end{equation}
indexed by $\textbf{i}\in I_l$ and $\alpha$ a $d$-dimensional multi-index with $|\alpha|\leq k$.

In our construction, we will approximate piecewise polynomial functions from $\mathcal{P}_k^l$ by deep ReLU neural networks. However, since a deep ReLU network can only represent a piecewise continuous function, this approximation will not be over the full cube $\Omega$. Rather, we will need to remove an arbitrarily small region from $\Omega$. This idea is from the method in \cite{shen2022optimal}, where this region was called the trifling region. Given $\epsilon > 0$ we define sets
\begin{equation}\label{shrunk-subcube-definition}
    \Omega^l_{\textbf{i},\epsilon} = \prod_{j=1}^d \begin{cases}
        [b^{-l}\textbf{i}_j,b^{-l}(\textbf{i}_j + 1) - \epsilon) & \textbf{i}_j < b^l-1\\
        [b^{-l}\textbf{i}_j,b^{-l}(\textbf{i}_j + 1)) & \textbf{i}_j = b^l-1,
    \end{cases}
\end{equation}
which are slightly shrunk sub-cubes (except at one edge) from \eqref{subcubes-definition}. We then define the good region to be
\begin{equation}
    \Omega_{l,\epsilon} := \bigcup_{\textbf{i}\in I_l} \Omega_{\textbf{i},\epsilon}^l.
\end{equation}
Next, we will show how to approximate piecewise polynomials from $\mathcal{P}_k^l$ on the set $\Omega_{l,\epsilon}$. For this, we begin with the following Lemma, which first appears in \cite{shen2022optimal}. This Lemma is essentially a minor modification of the bit-extraction technique used to prove Proposition \ref{bit-extractor-network-proposition}. We give a detailed proof for the reader's convenience in Appendix \ref{bit-extraction-appendix}.
\begin{lemma}\label{mapping-to-integers-lemma}
    Let $l \geq 0$ be an integer and $0 < \epsilon < b^{-l}$. Then there exists a deep ReLU neural network $q_d\in \Upsilon^{9d,2(b-1)l}(\mathbb{R}^d)$ such that
    \begin{equation}
        q_d(\Omega^l_{\textbf{\upshape i},\epsilon}) = \text{\upshape ind}(\textbf{\upshape i}) := \sum_{j=1}^db^{l(j-1)}\textbf{\upshape i}_j.
    \end{equation}
\end{lemma}
Note that here $\text{\upshape ind}(\textbf{\upshape i})\in \{0,...,b^{dl}-1\}$ is just an integer index corresponding to the sub-cube position $\textbf{i}$.

Using this Lemma we prove the following key technical Proposition, which shows how to efficiently approximate piecewise polynomial functions on the good set $\Omega_{l,\epsilon}$.
\begin{proposition}\label{one-level-deep-relu-approximation}
    Let $l\geq 0$ be an integer and $\epsilon > 0$. Suppose that $f\in \mathcal{P}_k^l$ is expanded in terms of the bases $\rho_{l,\textbf{\upshape i}}^\alpha$,
    \begin{equation}
        f(x) = \sum_{|\alpha|\leq k,~\textbf{\upshape i}\in I_l} a^\alpha_{\textbf{\upshape i}}\rho_{l,\textbf{\upshape i}}^\alpha(x).
    \end{equation}
    Let $1 \leq q \leq p \leq \infty$ and choose a parameter $\delta > 0$ and an integer $m \geq 1$. Then there exists a deep ReLU network $f_{\delta,m}\in \Upsilon^{22d+18,L}(\mathbb{R}^d)$ such that
    \begin{equation}
        \|f - f_{\delta,m}\|_{L_p(\Omega_{l,\epsilon})} \leq C\left(\delta\min\left\{1, b^{-dl}\delta^{-q}\right\}^{1/p} + 4^{-m}\right)\left(\sum_{|\alpha|\leq k,~\textbf{\upshape i}\in I_l} |a^\alpha_{\textbf{\upshape i}}|^q\right)^{1/q}
    \end{equation}
    (with the standard modification when $q = \infty$), and whose depth satisfies
    \begin{equation}
        L \leq C\begin{cases}
        m + l + \delta^{-q/2}\sqrt{1+dl\log(b)+q\log(\delta)}& \delta^{-q} \leq b^{dl}\\
        m + l + b^{dl/2}\sqrt{1 - \log{\delta} - (dl/q)\log(b)}& \delta^{-q} > b^{dl}.
        \end{cases}
    \end{equation}
    Here the constants $C := C(p,q,d,k,b)$ only depend upon $p,q,d,k$ and the base $b$, but not on $f$, $\delta$, $l$, $\epsilon$, or $m$.
\end{proposition}
Before we prove this Proposition, let us explain the intuition behind it and the meaning of the parameter $\delta$. The parameter $\delta$ represents a discretization level for the coefficients $a^\alpha_{\textbf{\upshape i}}$. Specifically, we will round each coefficient down (in absolute value) to the nearest multiple of $\delta$ to produce an approximation to $f$. Then, we will represent this approximation by encoding these discretized coefficients using deep ReLU networks. This reduces to encoding an integer vector which can be done optimally using Theorem \ref{sparse-approximation-theorem}. The two regimes $\delta^{-q} \leq b^{dl}$ and $\delta^{-q} > b^{dl}$ correspond to the case of \textit{dense} and \textit{sparse} coefficients, which are handled differently in Theorem \ref{sparse-approximation-theorem}.
\begin{proof}[Proof of Proposition \ref{one-level-deep-relu-approximation}]
    We begin by decomposing
    $
        f = \sum_{|\alpha|\leq k} f_\alpha
    $
    where
    \begin{equation}
        f_\alpha(x) = \sum_{\textbf{\upshape i}\in I_d} a^\alpha_{\textbf{\upshape i}}\rho_{l,\textbf{\upshape i}}^\alpha(x).
    \end{equation}
    By Proposition \ref{summing-networks} and the triangle inequality, it suffices to prove the result for each $f_\alpha$ individually with width $W = 20d+17$ (at the expense of larger constants). So in the following we assume that $f= f_\alpha$ and write $a_{\textbf{\upshape i}} := a^\alpha_{\textbf{\upshape i}}$.
    By normalizing $f$ we may assume also without loss of generality that
    \begin{equation}\label{coefficient-lq-bound-assumption}
        \left(\sum_{\textbf{\upshape i}\in I_l} |a_{\textbf{\upshape i}}|^q\right)^{1/q} \leq 1.
    \end{equation}

    We construct the following network. First, duplicate the input $x\in \mathbb{R}^d$ three times using an affine map
\begin{equation}
    x\rightarrow \begin{pmatrix}
            x\\
            x\\
            x
            \end{pmatrix}\in \Upsilon^{0}(\mathbb{R}^d,\mathbb{R}^{3d}).
\end{equation}
Next, we use Lemmas \ref{select-coordinates-lemma} and \ref{composition-lemma} to apply the network $q_d$ from Lemma \ref{mapping-to-integers-lemma} to the last coordinate and apply $q_1$ from Lemma \ref{mapping-to-integers-lemma} to each entry of the first coordinate to get
\begin{equation}
    x\rightarrow \begin{pmatrix}
            q_1(x_1)\\
            \vdots\\
            q_1(x_d)\\
            x\\
            q_d(x)
            \end{pmatrix}\in \Upsilon^{20d,2(b-1)l}(\mathbb{R}^d,\mathbb{R}^{2d+1}).
\end{equation}
We now compose with the affine map
\begin{equation}
    \begin{pmatrix}
            x\\
            y\\
            r
    \end{pmatrix}\rightarrow
    \begin{pmatrix}
            b^ly - x\\
            r
    \end{pmatrix}\in \Upsilon^{0}(\mathbb{R}^{2d+1},\mathbb{R}^{d+1}),
\end{equation}
where $x,y\in \mathbb{R}^d$ and $r\in \mathbb{R}$, to get
\begin{equation}\label{intermediate-map-1652}
    x\rightarrow \begin{pmatrix}
            b^lx_1 - q(x_1)\\
            \vdots\\
            b^lx_d - q(x_d)\\
            q_d(x)
            \end{pmatrix}\in \Upsilon^{20d,2(b-1)l}(\mathbb{R}^d,\mathbb{R}^{d+1}).
\end{equation}
On the set $\Omega^l_{\textbf{i},\epsilon}$ from \eqref{shrunk-subcube-definition} this map becomes
\begin{equation}
    x\rightarrow \begin{pmatrix}
            b^lx_1 - \textbf{i}_1\\
            \vdots\\
            b^lx_d - \textbf{i}_d\\
            \text{ind}(\textbf{i})
            \end{pmatrix}.
\end{equation}
The next step in the construction will be to approximate the coefficients $a_{\textbf{\upshape i}}$. To do this we round the $a_{\textbf{\upshape i}}$ down to the nearest multiple of $\delta$ (in absolute value) to get approximate coefficients
\begin{equation}
 \tilde{a}_{\textbf{\upshape i}} :=  \delta\sign(a_{\textbf{\upshape i}})\left\lfloor\frac{|a_{\textbf{\upshape i}}|}{\delta}\right\rfloor.
\end{equation}
We estimate the $\ell^p$-norm of the error this incurs as follows. Write
\begin{equation}
    \|a - \tilde{a}\|_{\ell^p} = \left(\sum_{\textbf{i}\in I_l} |a_{\textbf{\upshape i}} - \tilde{a}_{\textbf{\upshape i}}|^p\right)^{1/p}
\end{equation}
with the standard modification when $p=\infty$. Note that
\begin{equation}
\|a - \tilde{a}\|_{\ell^q} \leq \|a\|_{\ell^q} \leq 1
\end{equation}
by \eqref{coefficient-lq-bound-assumption}. In addition, it is clear from the rounding procedure that $\|a - \tilde{a}\|_{\ell^\infty} \leq \delta$.
H\"older's inequality thus implies that (since $p \geq q$)
\begin{equation}
    \|a - \tilde{a}\|_{\ell^p} \leq \|a - \tilde{a}\|^{q/p}_{\ell^q}\|a - \tilde{a}\|_{\ell^\infty}^{1-q/p} \leq \delta^{1-q/p}.
\end{equation}
On the other hand, using that $|I_l| = b^{dl}$, we can use the bound $\|a - \tilde{a}\|_{\ell^\infty} \leq \delta$ to get
\begin{equation}
    \|a - \tilde{a}\|_{\ell^p} \leq b^{dl/p}\delta.
\end{equation}
Putting these together, we get
\begin{equation}\label{coefficient-error-contribution}
    \|a - \tilde{a}\|_{\ell^p} \leq \delta\min\{b^{dl},\delta^{-q}\}^{1/p}.
\end{equation}
Next we construct a ReLU neural network which maps the index $\text{ind}(\textbf{i})$ to the rounded coefficients $\tilde{a}_{\textbf{i}}$. For this Theorem \ref{sparse-approximation-theorem} will be key. We set $N = b^{dl}$ and write $\tilde{a}_{\textbf{i}} = \delta \textbf{x}_{\text{ind}(\textbf{i})}$ for a vector $\textbf{x}\in \mathbb{Z}^N$ defined by
\begin{equation}
    \textbf{x}_{\text{ind}(\textbf{i})} = \sign(a_{\textbf{\upshape i}})\left\lfloor\frac{|a_{\textbf{\upshape i}}|}{\delta}\right\rfloor.
\end{equation}
We proceed to estimate $\|\textbf{x}\|_{\ell^1}$. We observe that by \eqref{coefficient-lq-bound-assumption}
\begin{equation}\label{x-lq-bound}
    \sum_{i=1}^N |\textbf{x}_i|^q \leq \sum_{\textbf{i}\in I_l} \left(\frac{|a_{\textbf{\upshape i}}|}{\delta}\right)^q \leq \delta^{-q}.
\end{equation}
Thus $\|\textbf{x}\|_{\ell^q} \leq \delta^{-1}$. Moreover, since $\textbf{x}\in \mathbb{Z}^N$, \eqref{x-lq-bound} implies that the number of non-zero entries in $\textbf{x}$ satisfies
\begin{equation}
    |\{i:\textbf{x}_i \neq 0\}| \leq \min\{\delta^{-q},N\}.
\end{equation}
We can thus use H\"older's inequality to get the bound
\begin{equation}
    \|\textbf{x}\|_{\ell^1} \leq |\{i:\textbf{x}_i \neq 0\}|^{1-1/q}\|\textbf{x}\|_{\ell^q}\leq \delta^{-1}\min\{\delta^{-q},N\}^{1-1/q}.
\end{equation}
Using this we apply Theorem \ref{sparse-approximation-theorem} with $M = \delta^{-1}\min\{\delta^{-q},N\}^{1-1/q}$ to the vector $\textbf{x}$. We calculate that if $\delta^{-q} \leq N$, then
\begin{equation}
    M = \delta^{-1}\delta^{-q(1-1/q)} = \delta^{-q} \leq N,
\end{equation}
while if $\delta^q > N$, then
\begin{equation}
    M = \delta^{-1}N^{(1-1/q)} = N(\delta^{-q}N^{-1})^{1/q} \geq N.
\end{equation}
Thus, Theorem \ref{sparse-approximation-theorem} (combined with a scaling by $\delta$) gives a network $g\in \Upsilon^{17,L}(\mathbb{R})$ such that $g(\text{ind}(\textbf{i})) = \tilde{a}_{\textbf{i}}$, whose depth is bounded by
\begin{equation}
    L \leq C\begin{cases}
        \delta^{-q/2}\sqrt{1+dl\log(b)+q\log(\delta)}& \delta^{-q} \leq b^{dl}\\
        b^{dl/2}\sqrt{1 - \log{\delta} - (dl/q)\log(b)}& \delta^{-q} > b^{dl}.
    \end{cases}
\end{equation}
Using Lemma \ref{select-coordinates-lemma} to apply $g$ to the last coordinate of the output in \eqref{intermediate-map-1652} gives a network $\tilde{f}_\delta\in \Upsilon^{20d+17,L}$
with depth bounded by
\begin{equation}
    L \leq 2(b-1)l + C\begin{cases}
        \delta^{-q/2}\sqrt{1+dl\log(b)+q\log(\delta)}& \delta^{-q} \leq b^{dl}\\
        b^{dl/2}\sqrt{1 - \log{\delta} - (dl/q)\log(b)}& \delta^{-q} > b^{dl},
        \end{cases}
\end{equation}
such that for $x\in \Omega^l_{\textbf{i},\epsilon}$ we have
\begin{equation}\label{tilde-f-delta-equation}
    \tilde{f}_\delta(x) = \begin{pmatrix}
            b^lx_1 - \textbf{i}_1\\
            \vdots\\
            b^lx_d - \textbf{i}_d\\
            \tilde{a}_{\textbf{i}}
            \end{pmatrix}.
\end{equation}
Finally, to obtain the network $f_{\delta,m}$ we use Lemma \ref{composition-lemma} to compose $\tilde{f}_\delta$ with a network $P_m$ which approximates the product
\begin{equation}
    \begin{pmatrix}
            z_1\\
            \vdots\\
            z_d\\
            z_{d+1}
            \end{pmatrix}\rightarrow z_{d+1}\prod_{j=1}^dz_j^{\alpha_j}
\end{equation}
on the set where $|z_j|\leq 1$ for all $j=1,...,d+1$. Note from the bound \eqref{coefficient-lq-bound-assumption} we see that $|\tilde{a}_{\textbf{i}}|\leq |a_{\textbf{i}}| \leq \|a\|_{\ell^q} \leq 1$. In addition, it is easy to see that for $x\in \Omega^l_{\textbf{i},\epsilon}$ we have $|b^lx_j - \textbf{i}_j| \leq 1$ for $j=1,...,d$. Thus the output of $\tilde{f}_\delta$ satisfies these assumptions for any $x\in \Omega_{l,\epsilon}$.

We construct the network $P_m$ using Proposition \ref{produt-network-proposition} as follows. Choose a parameter $m \geq 1$. We first approximate a function which multiplies the last entry $z_{d+1}$ by the $i$-th entry $z_i$. We do this by duplicating the $i$-th entry using an affine map and then applying Lemma \ref{select-coordinates-lemma} to apply the network $f_m$ from Proposition \ref{produt-network-proposition} to the $i$-th and last entries
\begin{equation}
    \begin{pmatrix}
            z_1\\
            \vdots\\
            z_d\\
            z_{d+1}
            \end{pmatrix}\rightarrow
            \begin{pmatrix}
            z_1\\
            \vdots\\
            z_d\\
            z_{d+1}\\
            z_i
            \end{pmatrix}\rightarrow
            \begin{pmatrix}
            z_1\\
            \vdots\\
            f_m(z_{d+1},z_i)
            \end{pmatrix}\in \Upsilon^{2d + 13,6m+3}(\mathbb{R}^{d+1},\mathbb{R}^{d+1}).
\end{equation}
In order to ensure that the resulting approximate product is still bounded in magnitude by $1$ (so that we can recursively apply these products), we apply the map $z\rightarrow \max(\min(z,-1),1)\in \Upsilon^{5,2}(\mathbb{R})$ to the last component. This gives a network $P^m_i\in \Upsilon^{2d+13,6m+5}(\mathbb{R}^{d+1},\mathbb{R}^{d+1})$, which maps
\begin{equation}
    P^m_i:\begin{pmatrix}
            z_1\\
            \vdots\\
            z_d\\
            z_{d+1}
            \end{pmatrix}\rightarrow
            \begin{pmatrix}
            z_1\\
            \vdots\\
            z_d\\
            z_{d+1}\\
            z_i
            \end{pmatrix}\rightarrow
            \begin{pmatrix}
            z_1\\
            \vdots\\
            \tilde{f}_m(z_{d+1},z_i)
            \end{pmatrix},
\end{equation}
where $\tilde{f}_m(z_{d+1},z_i) = \max(\min(f_m(z_{d+1},z_i),-1),1)$. Observe that since the true product $z_{d+1}z_i\in [-1,1]$ the truncation cannot increase the error, so that Proposition \ref{produt-network-proposition} implies
\begin{equation}
    |\tilde{f}_m(z_{d+1},z_i) - z_iz_{d+1}| \leq |f_m(z_{d+1},z_i) - z_iz_{d+1}| \leq 6\cdot 4^{-m}.
\end{equation}
We construct $P_m$ by composing (using Lemma \ref{composition-lemma}) $\alpha_j$ copies of $P^m_j$ and then applying an affine map which selects the last coordinate. Thus $P_m\in \Upsilon^{2d+13,L}(\mathbb{R}^{d+1})$ with $L \leq k(6m+5)$. Moreover, since all entries $z_i$ are bounded by $1$, we calculate that
\begin{equation}\label{product-error-contribution}
    \left|P_m(\textbf{z}) - z_{d+1}\prod_{j=1}^dz_j^{\alpha_j}\right| \leq \sum_{j=1}^d \alpha_j|\tilde{f}_m(z_{d+1},z_j) - z_jz_{d+1}| \leq 6k\cdot 4^{-m}.
\end{equation}
We obtain the network $f_{\delta,m}\in \Upsilon^{20d + 17,L}(\mathbb{R}^d,\mathbb{R})$ by composing $\tilde{f}_{\delta}$ and $P_m$ using Lemma \ref{composition-lemma}. Its depth is bounded by
\begin{equation}
    L \leq 2(b-1)l + k(6m+5) + C\begin{cases}
        \delta^{-q/2}\sqrt{1+dl\log(b)+q\log(\delta)}& \delta^{-q} \leq b^{dl}\\
        b^{dl/2}\sqrt{1 - \log{\delta} - (dl/q)\log(b)}& \delta^{-q} > b^{dl},
        \end{cases}
\end{equation}
and we note that $k(6m+5) \leq Cm$ for integers $m \geq 1$ and a constant $C := C(k)$ which depends upon $k$.

We bound the error using equations \eqref{coefficient-error-contribution}, \eqref{tilde-f-delta-equation}, \eqref{product-error-contribution}, and the fact that the basis $\rho^\alpha_{l,\textbf{i}}$ is normalized in $L_\infty$ and has disjoint support for fixed $\alpha$ to get
\begin{equation}
    \|f - f_{\delta,m}\|^p_{L_p(\Omega_{l,\epsilon})} \leq 2^{-ld}\sum_{\textbf{i}\in I_l} |a_{\textbf{i}} - \tilde{a}_{\textbf{i}}|^p + (6k\cdot 4^{-m})^p,
\end{equation}
so that
\begin{equation}
    \|f - f_{\delta,m}\|_{L_p(\Omega_{l,\epsilon})} \leq 2^{-ld/p}\|a - \tilde{a}\|_{\ell^p} + 6k\cdot 4^{-m} \leq C\left(\delta\min\left\{1, 2^{-dl}\delta^{-q}\right\}^{1/p} + 4^{-m}\right),
\end{equation}
which completes the proof.
\end{proof}
Next, we use the construction in Proposition \ref{one-level-deep-relu-approximation} to approximate a target function $f\in W^s(L_q(\Omega))$ in $L_p(\Omega)$ using deep ReLU neural networks, again removing an arbitrarily small trifling set in the spirit of \cite{shen2022optimal}.
\begin{proposition}\label{main-theorem-trifling-region}
	Let $\Omega = [0,1)^d$, $1\leq q\leq p\leq \infty$ and $f\in W^s(L_q(\Omega))$ with $\|f\|_{W^s(L_q(\Omega))} \leq 1$ for $s > 0$. Suppose that the Sobolev embedding condition is strictly satisfied, i.e.
	\begin{equation}\label{strict-sobolev-embedding-condition}
		\frac{1}{q} - \frac{1}{p} - \frac{s}{d} < 0,
	\end{equation}
	which guarantees the compact embedding $W^s(L_q(\Omega))\subset\subset L_p(\Omega)$ holds. Let $\epsilon > 0$ and $l_0 \geq 1$ be an integer and set $l^* = \lfloor\kappa l_0\rfloor$ with
	\begin{equation}
		\kappa := \frac{s}{s + d/p-d/q}.
	\end{equation}
	Note that $1\leq \kappa < \infty$ by the Sobolev embedding condition. Then there exists a network $f_{l_0,\epsilon}\in \Upsilon^{24d+20,L}(\mathbb{R}^d)$ such that
	\begin{equation}
		\|f - f_{l_0,\epsilon}\|_{L_p(\Omega_{l^*,\epsilon})} \leq Cb^{-sl_0}
	\end{equation}
	and whose depth is bounded by
	\begin{equation}
		L \leq Cb^{dl_0/2}.
	\end{equation}
	Here the constants $C := C(s,p,q,d,b)$ do not depend upon $l_0,f$ or $\epsilon$.
\end{proposition}

Before giving the detailed proof, let us comment on the intuition and the meaning of $\kappa$ and $l^*$. The idea is to decompose the function $f$ into different scales which consist of piecewise polynomial functions. We then appoximate these piecewise polynomial functions using neural networks via Proposition \ref{one-level-deep-relu-approximation} to varying degrees of accuracy dependent on the parameter $\delta$ used at each level. The parameter $l_0$ gives the finest level at which we approximate the coefficients in the dense regime $\delta^{-q} > b^{dl}$, while the level $l^*$ is the finest level which appears in the approximation. All levels between $l_0$ and $l^*$ are approximated in the sparse regime $\delta^{-q} \leq b^{dl}$. The parameter $\kappa$ controls the gap between $l_0$ and $l^*$ are essentially measures how adaptive the approximation must be. The proof is completed by choosing $\delta$ optimally at each level, analogous to the proof of the Birman-Solomyak Theorem \cite{birman1967piecewise} which calculates the metric entropy of the Sobolev unit ball.
\begin{proof}[Proof of Proposition \ref{main-theorem-trifling-region}]
	For a function $f\in L_q(\Omega)$, we write
	\begin{equation}
		\Pi_k^l(f) = \arg\min_{p\in \mathcal{P}_k^l} \|f - p\|_{L_q(\Omega)}
	\end{equation}
    for the $L_q$-projection of $f$ onto the space of piecewise polynomials of degree $k$. We will utilize the following well-known multiscale dyadic decomposition of the function $f$, which is a common tool in harmonic analysis \cite{birman1967piecewise,mallat1999wavelet,littlewood1931theorems} and the analysis of multigrid methods \cite{bramble1990parallel},
	\begin{equation}
		f = \sum_{l=0}^\infty f_l,
	\end{equation}
	where the components at level $l$ are defined by $f_0 = \Pi_k^0(f)$ and $f_l = \Pi_k^l(f) - \Pi_k^{l-1}(f)$ for $l \geq 1$. Expanding the components $f_l$ in the basis $\rho^\alpha_{l,\textbf{\upshape i}}$, we write
    \begin{equation}\label{expansion-of-fl}
        f_l(x) = \sum_{|\alpha|\leq k,~\textbf{\upshape i}\in I_l} a^\alpha_{l,\textbf{\upshape i}}\rho_{l,\textbf{\upshape i}}^\alpha(x).
    \end{equation}
    The key estimate in the proof is to establish the following coefficient bound
    \begin{equation}\label{bramble-hilbert-coefficient-bound}
        |a^\alpha_{l,\textbf{\upshape i}}| \leq Cb^{(d/q - s)l}|f|_{W^s(L_q(\Omega_{\textbf{i}^-}^{l-1}))},
    \end{equation}
    where $\Omega_{\textbf{i}^-}^{l-1} \supset \Omega_{\textbf{i}}^{l}$ is the parent domain of $\Omega_{\textbf{i}}^l$ when $l\geq 1$. When $l=0$, we have the simple modification 
    $$|a^\alpha_{0,\textbf{\upshape 0}}| \leq C\|f\|_{W^s(L_q(\Omega))}.$$
    
    We prove \eqref{bramble-hilbert-coefficient-bound} by utilizing the Bramble-Hilbert lemma \cite{bramble1970estimation} and a well-known scaling argument. For $l\geq 1$ consider the scaling map $S_{l,\textbf{i}}$ which scales the small domain $\Omega_{\textbf{i}^-}^{l-1}$ up to the large domain $\Omega$, defined by
    \begin{equation}
        S_{l,\textbf{i}}(f)(x) = f(b^{l-1}x - \textbf{i}^{-})\in L_q(\Omega)
    \end{equation}
    for $f\in L_q(\Omega_{\textbf{i}^-}^{l-1})$. We verify the following simple facts
    \begin{equation}\label{scaling-simple-facts}
    \begin{split}
        &|S_{l,\textbf{i}}(f)|_{W^s(L_q(\Omega))} = b^{(d/q-s)(l-1)}|f|_{W^s(L_q(\Omega_{\textbf{i}^-}))}\\
        &S_{l,\textbf{i}}(f_l) = S_{l,\textbf{i}}(\Pi_k^l(f) - \Pi_k^{l-1}(f)) = \Pi_k^1(S_{l,\textbf{i}}(f)) - \Pi_k^0(S_{l,\textbf{i}}(f))\\
        &S_{l,\textbf{i}}(\rho_{l,\textbf{\upshape i}}^\alpha) = \rho_{1,\textbf{\upshape j}}^\alpha,
    \end{split}
    \end{equation}
    where $\textbf{j}\in \{0,1,...,b\}^d$ is the index of $\Omega^l_\textbf{i}$ in $\Omega^{l-1}_{\textbf{i}^-}$, i.e. $\textbf{j}\equiv \textbf{i} \pmod{b}$. From the last two facts we deduce that 
    \begin{equation}
        \Pi_k^1(S_{l,\textbf{i}}(f)) - \Pi_k^0(S_{l,\textbf{i}}(f)) = \sum_{\textbf{j}\in I_1} a^\alpha_{l,(b\textbf{i}^- + \textbf{j})}\rho_{1,\textbf{\upshape j}}^\alpha,
    \end{equation}
    where the $a^\alpha_{l,(b\textbf{i}^- + \textbf{j})}$ are the coefficients from the expansion \eqref{expansion-of-fl} of $f_l$. Combining this with the first fact from \eqref{scaling-simple-facts}, it suffices to prove \eqref{bramble-hilbert-coefficient-bound} when $l=1$ and apply this to $S_{l,\textbf{i}}(f)$.

    To prove \eqref{bramble-hilbert-coefficient-bound} when $l=1$, we use the Bramble-Hilbert lemma \cite{bramble1970estimation}. We calculate using the Bramble-Hilbert lemma that
    \begin{equation}
    \begin{split}
        &\|\Pi_k^0(f) - f\|_{L_q(\Omega^1_\textbf{i})} \leq \|\Pi_k^0(f) - f\|_{L_q(\Omega)} \leq C|f|_{W^s(L_q(\Omega))}\\
        &\|\Pi_k^1(f) - f\|_{L_q(\Omega^1_\textbf{i})} \leq C|f|_{W^s(L_q(\Omega^1_\textbf{i}))} \leq C|f|_{W^s(L_q(\Omega))}.
    \end{split}
    \end{equation}
    Combining these two estimates, we get
    \begin{equation}
        \|\Pi_k^0(f) - \Pi_k^1(f)\|_{L_q(\Omega^1_\textbf{i})} \leq C|f|_{W^s(L_q(\Omega))}.
    \end{equation}
    When $l=0$ we make the modification
    \begin{equation}
        \|\Pi_k^0(f)\|_{L_q(\Omega)} \leq \|f\|_{L_q(\Omega)} + \|\Pi_k^0(f) - f\|_{L_q(\Omega)} \leq C\|f\|_{W^s(L_q(\Omega))}.
    \end{equation}
    Now we use the fact that all norms on the finite dimensional space of polynomials of degree at most $k$ are equivalent to transfer the $L_q$ bound to a bound on the coefficients. This implies \eqref{bramble-hilbert-coefficient-bound}. 

    From \eqref{bramble-hilbert-coefficient-bound}, we deduce the following bound on the $\ell^q$-norm of the coefficients of $f_l$:
    \begin{equation}\label{coefficient-l-q-bound}
        \left(\sum_{|\alpha|\leq k,~\textbf{\upshape i}\in I_l} |a^\alpha_{l,\textbf{\upshape i}}|^q\right)^{1/q} \leq Cb^{(d/q - s)l}\left(\sum_{|\alpha|\leq k,~\textbf{\upshape i}\in I_l} |f|^q_{W^s(L_q(\Omega_{\textbf{i}^-}^{l-1}))}\right)^{1/q} \leq Cb^{(d/q - s)l},
    \end{equation}
    since $\|f\|_{W^s(L_q(\Omega))} \leq 1$. This follows from the sub-additivity of the Sobolev norm,
    \begin{equation}\label{sobolev-subadditivity}
    \sum_{\textbf{i}\in I_{l-1}} |f|^q_{W^s(L_q(\Omega_{\textbf{i}}^{l-1}))} \leq |f|^q_{W^s(L_q(\Omega))},
    \end{equation}
    since each $\Omega_{\textbf{i}^-}^{l-1}$ appears a finite number of times  in the sum \eqref{coefficient-l-q-bound} (namely $\binom{k+d}{d}b^d$ which is independent of $l$). 
    
    We remark that the Sobolev sub-additivity \eqref{sobolev-subadditivity} immediately follows from the definitions \eqref{sobolev-semi-norm-definition} and \eqref{sobolev-semi-norm-definition-fractional}. Note also that the bound \eqref{coefficient-l-q-bound} also easily follows when the standard modifications are made for $q=\infty$. 
    
    Next, we derive the following bound, which follows from \eqref{coefficient-l-q-bound}, the $L_\infty$-normalization of the basis functions $\rho^\alpha_{l,\textbf{\upshape i}}$, the fact that for fixed $\alpha$ the functions $\rho^\alpha_{l,\textbf{\upshape i}}$ have disjoint support, and the assumption that $p \geq q$:
    \begin{equation}\label{l-p-tail-bounds}
        \|f_l\|_{L_p(\Omega)} \leq \sum_{|\alpha|\leq k}b^{-dl/p}\left(\sum_{\textbf{\upshape i}\in I_l} |a^\alpha_{l,\textbf{\upshape i}}|^p\right)^{1/p} \leq b^{-dl/p}\binom{k+d}{d}^{1-1/p}\left(\sum_{|\alpha|\leq k,~\textbf{\upshape i}\in I_l} |a^\alpha_{l,\textbf{\upshape i}}|^p\right)^{1/p} \leq Cb^{(d/q - d/p - s)l}.
    \end{equation}

    We now complete the proof by using Proposition \ref{one-level-deep-relu-approximation} to approximate each $f_l$ for $l=1,...,l^*$, for which we must choose appropriate parameters. First, we choose $\tau > 0$ such that 
    $$
    \frac{d}{q} - \frac{d}{p} - s + \left(1 - \frac{q}{p}\right)\tau < 0.
    $$
    Note that this condition can be satisfied since $q\leq p$ and the Sobolev embedding condition \eqref{strict-sobolev-embedding-condition} holds. For each level $l$ we choose parameters
    \begin{equation}\label{delta-choice-nn}
	\delta = \delta(l) = \begin{cases}
	b^{-dl_0/q + \tau(l-l_0)} & l\geq l_0\\
	b^{-dl/q + (s+1)(l-l_0)} & l < l_0
	\end{cases}
\end{equation}
and
\begin{equation}\label{m-choice-nn}
    m = m(l) = K_1 l_0 + K_2 l
\end{equation}
in Proposition \ref{one-level-deep-relu-approximation}, where $K_1,K_2 > 0$ are parameters to be chosen later. 
Note that $\delta(l)^{-q} \leq b^{dl}$ when $l\geq l_0$ and $\delta(l)^{-q} > b^{dl}$ when $l < l_0$. This means that the coarser levels are discretized finely and the coefficients are dense, while the finer levels are discretized coarsely so that the coefficients are sparse.

So, we define the network $f_{l_0,\epsilon}$ using Proposition \ref{summing-networks} to be
\begin{equation}
    f_{l_0,\epsilon} = \sum_{l=0}^{l^*} f_{\delta(l),m(l)},
\end{equation}
where $f_{\delta(l),m(l)}$ is constructed using Proposition \ref{one-level-deep-relu-approximation} applied to $f_l$ with parameters $\delta = \delta(l)$ and $m = m(l)$. 

Propositions \ref{summing-networks} and \ref{one-level-deep-relu-approximation} imply that $f_{l_0,\epsilon}\in \Upsilon^{24d+20,L}(\mathbb{R}^d)$ with
\begin{equation}
    L\leq \sum_{l=0}^{l^*} L_l,
\end{equation}
where the depths $L_l$ for each level are bounded using Proposition \ref{one-level-deep-relu-approximation} by
\begin{equation}
    L_l \leq C\begin{cases}
        m(l) + l + \delta(l)^{-q/2}\sqrt{1+dl\log(b)+q\log(\delta(l))}& \delta(l)^{-q} \leq b^{dl}\\
        m(l) + l + b^{dl/2}\sqrt{1 - \log{\delta(l)} - (dl/q)\log(b)}& \delta(l)^{-q} > b^{dl}.
        \end{cases}
\end{equation}
Plugging in the expressions for $\delta(l)$ and $m(l)$ given in \eqref{delta-choice-nn} and \eqref{m-choice-nn}, and using that $\delta(l)^{-q} \leq b^{dl}$ when $l\geq l_0$ and $\delta(l)^{-q} > b^{dl}$ when $l < l_0$, we get the bound
\begin{equation}
\begin{split}
    L \leq C\left(\sum_{l=0}^{l^*} K_1 l_0 +\left(K_2+1\right)l\right. &+ b^{dl_0/2}\sum_{l=0}^{l_0 - 1}b^{d(l - l_0)/2}\sqrt{1 + \log(b)(s+1)(l-l_0)}\\ &+ \left.b^{dl_0/2}\sum_{l=l_0}^{l^*} b^{-\tau(l-l_0)/2}\sqrt{1 + \log(b)(d/q + \tau)(l-l_0)}\right).
\end{split}
\end{equation}
Summing the series above (and noting that the latter two are bounded by convergent geometric series), we get
\begin{equation}
    L \leq C((l^*)^2 + b^{dl_0/2}).
\end{equation}
Note that here the constant $C$ depends upon the choice of parameters $K_1$ and $K_2$.
Since $l^* \leq \kappa l_0$ is a linear function of $l_0$, the quadratic term $(l^*)^2 \leq (\kappa l_0)^2$ is dominated by the exponential second term. Thus we get $L \leq Cb^{dl_0/2}$ (for a potentially larger constant $C$).

Finally, we bound the error. For this we use Proposition \ref{one-level-deep-relu-approximation} to bound 
\begin{equation}
    \|f_{\delta(l),m(l)} - f_l\|_{L_p(\Omega_{l,\epsilon})} \leq C\left(\delta(l)\min\left\{1, b^{-dl}\delta(l)^{-q}\right\}^{1/p} + 4^{-m(l)}\right)\left(\sum_{|\alpha|\leq k,~\textbf{\upshape i}\in I_l} |a^\alpha_{l,\textbf{\upshape i}}|^q\right)^{1/q}.
\end{equation}
Combining this with the bound \eqref{coefficient-l-q-bound}, plugging in the choices \eqref{delta-choice-nn} and \eqref{m-choice-nn} (here again we have $b^{-dl}\delta(l)^{-q} \leq 1$ when $l \geq l_0$ and $b^{-dl}\delta(l)^{-q} > 1$ when $l < l_0$), and noting that $\Omega_{l,\epsilon} \supset \Omega_{l^*,\epsilon}$ if $l \leq l^*$, we get
\begin{equation}
\begin{split}
    \sum_{l=0}^{l^*} \|f_{\delta(l),m(l)} - f_l\|_{L_p(\Omega_{l^*,\epsilon})} \leq C&\left(\sum_{l=0}^{l_0 - 1}b^{(d/q-s)l}\left[\delta(l) + 4^{-m(l)}\right]\right. + \\&~~~\left.\sum_{l=l_0}^{l^*}b^{(d/q-s)l}\left[\delta(l)^{1 - q/p}b^{-(d/p)l} + 4^{-m(l)}\right]\right).
\end{split}
\end{equation}
Plugging in our choices for $\delta(l)$ and $m(l)$, we calculate
\begin{equation}\label{bound-on-error-sum-2177}
\begin{split}
    \sum_{l=0}^{l^*} \|f_{\delta(l),m(l)} - f_l\|_{L_p(\Omega_{l^*,\epsilon})} \leq C&\left(b^{-sl_0}\sum_{l=0}^{l_0 - 1}b^{l-l_0}\right.\\
    &~~~~+ b^{-sl_0}\sum_{l=l_0}^{l^*} b^{(d/q-d/p-s+\tau(1-q/p))(l-l_0)}\\
    &~~~\left.+ \sum_{l=0}^{l^*}b^{(d/q-s)l}4^{-K_1l_0-K_2l}\right).
\end{split}
\end{equation}
The first sum above is a convergent geometric series and is bounded by $Cb^{-sl_0}$. Due to the choice of $\tau$, the second sum is also a convergent gemoetric series, and is also bounded by $Cb^{-sl_0}$. Choosing $K_1$ and $K_2$ large enough so that $4^{-K_1} \leq b^{-s}$ and $4^{-K_2} < b^{(s-d/q)}$, the final sum is also a convergent geometric series which is bounded by $Cb^{-sl_0}$. Thus, we obtain
\begin{equation}
    \sum_{l=0}^{l^*} \|f_{\delta(l),m(l)} - f_l\|_{L_p(\Omega_{l^*,\epsilon})} \leq Cb^{-sl_0}
\end{equation}
for an appropriate constant $C$. Finally, we estimate
\begin{equation}
\begin{split}
    \|f - f_{l_0,\epsilon}\|_{L_p(\Omega_{l^*,\epsilon})} \leq \sum_{l=0}^{l^*} \|f_{\delta(l),m(l)} - f_l\|_{L_p(\Omega_{l^*,\epsilon})} + \sum_{l=l^*+1}^\infty \|f_l\|_{L_p(\Omega_{l^*,\epsilon})}.
\end{split}
\end{equation}
Utilizing \eqref{bound-on-error-sum-2177} and \eqref{l-p-tail-bounds} we get
\begin{equation}
    \|f - f_{l_0,\epsilon}\|_{L_p(\Omega_{l^*,\epsilon})} \leq Cb^{-sl_0} + C\sum_{l=l^*+1}^\infty b^{(d/q - d/p - s)l}.
\end{equation}
The compact Sobolev embedding condition implies that the second sum is a convergent geometric series, bounded by a multiple of its first term. This gives
\begin{equation}
    \|f - f_{l_0,\epsilon}\|_{L_p(\Omega_{l^*,\epsilon})} \leq C(b^{-sl_0} + b^{(d/q - d/p - s)l^*}).
\end{equation}
Finally, we use the definition of $l^*$ and $\kappa$ to see that
\begin{equation}
    b^{(d/q - d/p - s)l^*} \leq Cb^{-sl_0},
\end{equation}
which completes the proof.
\end{proof}
We note that a completely analogous Proposition holds for the Besov spaces $B^s_r(L_q(\Omega))$, i.e. Proposition \ref{main-theorem-trifling-region} holds with the Sobolev space $W^s(L_q(\Omega))$ replaced by $B^s_r(L_q(\Omega))$. The proof is exactly the same, utilizing a piecewise polynomial approximation, with the main difference being that the Bramble-Hilbert lemma is replaced by the following bound on piecewise polynomial approximation of Besov functions, known as Whitney's theorem (see \cite{devore1988interpolation}, Section 3 for instance)
\begin{equation}
    \|\Pi_{k-1}^0(f) - f\|_{L_q(\Omega)} \leq C\omega_k(f,1)_q,
\end{equation}
where $\Omega = [0,1]^d$ is the unit cube, $\omega_k$ is the modulus of smoothness introduced in \eqref{definition-of-modulus-of-smoothness}, and the constant depends upon $d,q$ and $k$. The proof proceeds via the same scaling argument, with the sub-additivity \eqref{sobolev-subadditivity} replaced by the corresponding result for the modulus of smoothness
$$
\sum_{\textbf{i}\in I_{l-1}} \omega_k(f,t,\Omega_{\textbf{i}}^{l-1})^q_q \leq C\omega_k(f,t)^q_q.
$$
Note that here the modulus of smoothness on the left hand side is taken relative to subdomain $\Omega_{\textbf{i}}^{l-1}$, while on the right the modulus is taken relative to the whole domain $\Omega$. This inequality holds for a constant $C$ depending on $d,q$ and $k$ (see \cite{devore1993besov}).
Finally, we use the bound
\begin{equation}
    \omega_k(f,t)_q \leq Ct^s\|f\|_{B^s_r(L_q(\Omega))},
\end{equation}
which holds as long as $k > s$, to obtain a bound on the error in terms of the Besov norm from a bound in terms of the modulus of smoothness.

Finally, we show how to remove the trifling region to give a proof of Theorem \ref{deep-network-upper-bound-theorem}. This is a technical construction similar to the method in \cite{shen2022optimal,lu2021deep,shijun2021deep}, but we significantly reduce the size of the required network (in particular the width no longer depends exponentially on the input dimension) by using the sorting network construction from Corollary \ref{order-statistic-network-corollary}. 

In addition to using sorting network, in our approach we use different bases $b_i$ to create minimally overlapping trifling regions. This is somewhat different than the aforementioned approaches \cite{shen2022optimal,lu2021deep,shijun2021deep}, which shift the grid to achieve the same effect. The reason we did this is to avoid the use of Sobolev and Besov extension theorems, as our method allows everything to stay within the unit cube. Although such extension theorems could be used, they become quite technical in full generality, and so we have found our approach simpler.

The proof of Theorem \ref{deep-network-upper-bound-theorem-besov} is completely analogous using Proposition \ref{main-theorem-trifling-region} with the Sobolev spaces replaced by Besov spaces, and is omitted.
\begin{proof}[Proof of Theorem \ref{deep-network-upper-bound-theorem}]
    We assume without loss of generality that $f\in W^s(L_q(\Omega))$ has been normalized, i.e. so that $\|f\|_{W^s(L_q(\Omega))} \leq 1$.
    
    In order to remove the trifling region from the preceding construction we will make use of different bases $b$. Let $r$ be the smallest integer such that $2^r \geq 2d + 2$ (so that $2^r \leq 4d + 4$), set $m = 2^r$, and set $b_i = \pi_i$ (the $i$-th prime number) for $i=1,...,m$.

    Let $n \geq b_m$ be an integer. We will construct a network $f_L\in \Upsilon^{30d+24,L}(\mathbb{R}^d)$ such that
    \begin{equation}
        \|f - f_L\|_{L_p(\Omega)} \leq Cn^{-s}
    \end{equation}
    with depth $L \leq Cn^{d/2}$, which will complete the proof.

    For $i=1,...,m$, set $l_i = \lfloor\log(n)/\log(b_i)\rfloor$ to be the largest power of $b_i$ which is at most $n$, and write $l_i^* = \lfloor \kappa l_i\rfloor$ where $\kappa$ is defined as in Proposition \ref{main-theorem-trifling-region}.
    Note that since the $\pi_i$ are all pairwise relatively prime, the numbers
    \begin{equation}
        S := \left\{\frac{1}{\pi_1^{l^*_1}},...,\frac{\pi_1^{l^*_1}-1}{\pi_1^{l^*_1}},\frac{1}{\pi_2^{l^*_2}},...,\frac{\pi_2^{l^*_2}-1}{\pi_2^{l^*_2}},...,\frac{1}{\pi_m^{l^*_m}},...,\frac{\pi_m^{l^*_m}-1}{\pi_m^{l^*_m}}\right\}
    \end{equation}
    are all distinct. Choose an $\epsilon > 0$ which satisfies
    \begin{equation}\label{epsilon-condition}
        \epsilon < \min_{x\neq y\in S} |x-y|,
    \end{equation}
    i.e. which is smaller than the distance between the two closest elements of $S$. This $\epsilon$ has the property that any $x\in [0,1]$ is contained in at most one of the sets
    \begin{equation}\label{bad-sets-equation}
        [j\pi_i^{-l^*_i}-\epsilon,j\pi_i^{-l^*_i})~\text{for}~i=1,...,m~\text{and}~j=1,...,\pi_i^{l^*_i}-1.
    \end{equation}
    This means that for any $x\in \Omega$, we have
    $
        x\notin \Omega_{l_i^*,\epsilon}
    $
    for at most $d$ different values $i$. Here $\Omega_{l_i^*,\epsilon}$ is the good region at level $l_i^*$ with base $b_i$. This holds since  $x$ has $d$ coordinates and each coordinate can be contained in at most one bad set from \eqref{bad-sets-equation}.

    We now use Proposition \ref{main-theorem-trifling-region}, setting $l_0 = l_i$ and using an $\epsilon$ satisfying \eqref{epsilon-condition}, to construct $f_i\in \Upsilon^{24d+20,L}(\mathbb{R}^d)$ which satisfies
    \begin{equation}
        \|f - f_i\|_{L_p(\Omega_{l_i^*,\epsilon})} \leq C\pi_i^{-sl_i} \leq Cn^{-s}
    \end{equation}
    and has depth bounded by
    \begin{equation}
        L\leq C\pi_i^{dl_i/2} \leq Cn^{d/2}.
    \end{equation}

    Finally, we construct the following network. We sequentially duplicate the input and apply the network $f_i$ to the new copy using Lemma \ref{select-coordinates-lemma} to get
    \begin{equation}\label{apply-all-f-i-network}
        x\rightarrow \begin{pmatrix}
            x\\
            x
        \end{pmatrix}\rightarrow \begin{pmatrix}
            x\\
            f_1(x)
        \end{pmatrix}\rightarrow \begin{pmatrix}
            x\\
            x\\
            f_1(x)
        \end{pmatrix}\rightarrow \begin{pmatrix}
            x\\
            f_2(x)\\
            f_1(x)
        \end{pmatrix}\rightarrow \cdots \rightarrow \begin{pmatrix}
            f_m(x)\\
            \vdots\\
            f_2(x)\\
            f_1(x)
        \end{pmatrix}\in \Upsilon^{30d+24,L}(\mathbb{R}^d,\mathbb{R}^m)
    \end{equation}
    with $L\leq C\sum_{i=1}^m\pi_i^{dl_i/2} \leq Cn^{d/2}.$

    We construct the network $f_L\in \Upsilon^{30d+24,L}(\mathbb{R}^d)$ by composing the network from \eqref{apply-all-f-i-network} with the order statistic network which selects the median, i.e. the $m/2$-largest value. By construction the network depth of $f_L$ satisfies
    \begin{equation}
        L\leq Cn^{d/2} + \binom{m+1}{2} \leq Cn^{d/2},
    \end{equation}
    since $\binom{m+1}{2}$ is a constant independent of $n$.

    To bound the approximation error of $f_L$ we introduce the following notation. Given $x\in [0,1)^d$, we write
    \begin{equation}
        \mathcal{K}(x) = \{i:~x\in \Omega_{l_i^*,\epsilon}\}
    \end{equation}
    for the set of indices such that $x$ is contained in the good region for the base $b_i$ decomposition. Since $x$ fails to be in $\Omega_{l_i^*,\epsilon}$ for at most $d$ values of $i$, we get
    \begin{equation}
        |\mathcal{K}(x)| \geq m - d \geq m/2 + 1
    \end{equation}
    since $m \geq 2d+2$. Thus the $m/2$-largest element among the $f_1(x),...,f_m(x)$ is both smaller and larger than some element of $\{f_i(x),~i\in \mathcal{K}(x)\}$, which implies
    \begin{equation}
    	\min_{i\in \mathcal{K}(x)}f_i(x) \leq f_L(x) \leq \max_{i\in \mathcal{K}(x)}f_i(x),
    \end{equation}
    so that
    \begin{equation}
    	|f_L(x) - f(x)| \leq \max_{i\in \mathcal{K}(x)} |f_i(x) - f(x)|.
    \end{equation}
    This completes the proof when $p=\infty$, since if $i\in \mathcal{K}(x)$ then $|f_i(x) - f(x)|\leq Cn^{-s}$ by Proposition \ref{main-theorem-trifling-region} and the definition of $\mathcal{K}(x)$.

    For $p < \infty$, we note that
    \begin{equation}
    \begin{split}
    	\int_{\Omega} |f_L(x) - f(x)|^p dx \leq \int_{\Omega} \max_{i\in \mathcal{K}(x)} |f_i(x) - f(x)|^pdx &\leq \int_{\Omega}\sum_{i\in \mathcal{K}(x)} |f_i(x) - f(x)|^pdx\\
	&\leq \sum_{i=1}^{m} \|f_i - f\|_{L_p(\Omega_{l^*_i,\epsilon})}^p\\
	&\leq Cn^{-sp}.
    \end{split}
    \end{equation}
    Taking $p$-th roots completes the proof.
\end{proof}

\section{Lower Bounds}\label{lower-bounds-section}
In this section, we study lower bounds on the approximation rates that deep ReLU neural networks can achieve on Sobolev spaces. Our main result is to prove Theorem \ref{deep-network-lower-bound-theorem}, which shows that the construction of Theorem \ref{deep-network-upper-bound-theorem} is optimal in terms of the number of parameters. In addition, we show that the representation of sparse vectors proved in Theorem \ref{sparse-approximation-theorem} is optimal. 

The key concept is the notion of VC dimension, which was used in \cite{yarotsky2018optimal,shen2022optimal} to prove lower bounds for approximation in the $L_\infty$-norm. We generalize these results to obtain sharp lower bounds on the approximation in $L_p$ as well.
Let $K$ be a class of functions defined on $\mathbb{R}^d$. The VC-dimension \cite{vapnik2015uniform} of $K$ is defined to be the largest number $n$ such that there exists a set of points $x_1,...,x_n\in \Omega$ such that
\begin{equation}
    |\{(\sign(g(x_1)),...,\sign(g(x_n))),~g\in K\}| = 2^n,
\end{equation}
i.e. such that every sign pattern at the points $x_1,...,x_n$ can be matched by a function from $K$. Such a set of points is said to be shattered by $K$.

The VC dimension of classes of functions defined by neural networks has been extensively studied and the most precise results are available for piecewise polynomial activation functions. We will discuss two main results concerning the VC dimension of $\Upsilon^{W,L}(\mathbb{R}^d)$. The first bound is most useful when the depth $L$ is fixed and the width $W$ is large and is given by
\begin{equation}\label{VC-dimension-bound-1}
    \text{VC-dim}(\Upsilon^{W,L}(\mathbb{R}^d)) \leq C(W^2L^2\log(WL)).
\end{equation}
This was proved in Theorem 6 of \cite{bartlett2019nearly}. The second bound, which is most informative when the width $W$ is fixed and the depth $L$ is large is
\begin{equation}\label{VC-dimension-bound}
    \text{VC-dim}(\Upsilon^{W,L}(\mathbb{R}^d)) \leq C(W^3L^2).
\end{equation}
This was proved in Theorem 8 of \cite{bartlett2019nearly} using a technique developed in \cite{goldberg1993bounding}. In either case, the VC-dimension of a deep ReLU neural network with $P = O(W^2L)$ parameters is bounded by $CP^2$, with this bound achieved up to a constant only in the case where the width $W$ is fixed and the depth $L$ grows. This bound on the VC-dimension was used in \cite{yarotsky2018optimal,shen2022optimal} to prove Theorem \ref{deep-network-lower-bound-theorem} in the case $p=\infty$. However, in order to extend the lower bound to $p < \infty$ a more sophisticated analysis is required. The key argument is captured in the following Proposition.
\begin{theorem}\label{main-theorem}
    Let $p > 0$, $\Omega = [0,1]^d$ and suppose that $K$ is a translation invariant class of functions whose VC-dimension is at most $n$. By translation invariant we mean that $f\in K$ implies that $f(\cdot - v)\in K$ for any fixed vector $v\in \mathbb{R}^d$. Then there exists an $f\in W^s(L_\infty(\Omega)) \cap B^s_1(L_\infty(\Omega))$ such that
    \begin{equation}
    \inf_{g\in K} \|f - g\|_{L_p(\Omega)} \geq C(p,d,s)n^{-\frac{s}{d}} \max\left\{\|f\|_{W^s(L_\infty(\Omega))}, \|f\|_{B^s_1(L_\infty(\Omega))}\right\}.
\end{equation}
\end{theorem}
Although the translation invariance holds for many function classes of interest, it is an interesting problem whether it can be removed. Before proving this result, we first show how Theorem \ref{deep-network-lower-bound-theorem} follows from this.
\begin{proof}[Proof of Theorem \ref{deep-network-lower-bound-theorem}]
	Note that the class of deep ReLU networks $\Upsilon^{W,L}(\mathbb{R}^d)$ is translation invariant. Combining this with the VC-dimension bounds \eqref{VC-dimension-bound-1} and \eqref{VC-dimension-bound}, Theorem \ref{main-theorem} implies Theorem \ref{deep-network-lower-bound-theorem} in the case $q = \infty$ and $r=1$. The general case follows trivially since $W^s(L_\infty(\Omega))\subset W^s(L_q(\Omega))$ for any $q \leq \infty$, and $B^s_1(L_\infty(\Omega) \subset B^s_r(L_q(\Omega))$ for $r \geq 1$ and $q\leq \infty$.
\end{proof}

Let us turn to the proof of Theorem \ref{main-theorem}. A key ingredient is the well-known Sauer-Shelah lemma \cite{sauer1972density,shelah1972combinatorial}.
\begin{lemma}[Sauer-Shelah Lemma]\label{sauer-shelah-lemma}
    Suppose that $K$ has VC-dimension at most $n$. Given any collection of $N$ points $x_1,...,x_N\in \Omega$, we have
    \begin{equation}\label{sauer-shelah-bound}
        |\{(\sign(g(x_1)),...,\sign(g(x_N))),~g\in K\}| \leq \sum_{i=0}^n\binom{N}{i}.
    \end{equation}
\end{lemma}

We will also utilize the following elementary bound on the size of a Hamming ball.
\begin{lemma}\label{binomial-sum-bound}
    Suppose that $N \geq 2n$, then
    \begin{equation}
        \sum_{i=0}^n\binom{N}{i} \leq 2^{NH(n/N)},
    \end{equation}
    where $H(p)$ is the entropy function
    \begin{equation}
        H(p) = -p\log(p)-(1-p)\log(1-p).
    \end{equation}
    (Note that all logarithms here are taken base $2$.)
\end{lemma}
\begin{proof}
    Observe that since $N - n \geq n$, we have
    \begin{equation}
    \begin{split}
        \left(\frac{N-n}{N}\right)^{N-n}\left(\frac{n}{N}\right)^n\sum_{i=0}^n\binom{N}{i} &\leq \sum_{i=0}^n\binom{N}{i}\left(\frac{N-n}{N}\right)^{N-i}\left(\frac{n}{N}\right)^i\\
        & < \sum_{i=0}^N\binom{N}{i}\left(\frac{N-n}{N}\right)^{N-i}\left(\frac{n}{N}\right)^i = 1.
    \end{split}
    \end{equation}
    This means that
    \begin{equation}
        \sum_{i=0}^n\binom{N}{i} \leq \left[\left(\frac{N-n}{N}\right)^{N-n}\left(\frac{n}{N}\right)^n\right]^{-1}.
    \end{equation}
    Taking logarithms, we obtain
    \begin{equation}
        \log\left(\sum_{i=0}^n\binom{N}{i}\right) \leq -N\left((N-n)\log\left(\frac{N-n}{N}\right) + n\log\left(\frac{n}{N}\right)\right) = NH(n/N)
    \end{equation}
    as desired.
\end{proof}
Utilizing these lemmas, we give the proof of Theorem \ref{main-theorem}.
\begin{proof}[Proof of Theorem \ref{main-theorem}]
    Let $c < 1/2$ be chosen so that $H(c) < 1/2$ (for instance $c = 0.1$ will work) and fix $$k := \lceil \sqrt[d]{n/c}\rceil \leq C(d)n^{1/d},$$
    and $\epsilon := k^{-1}$.
    Next, we consider shifts of an equally spaced grid with side length $\epsilon$. Specifically, for each $\lambda\in [0,\epsilon)^d$, define the point set
    \begin{equation}
        X_{\lambda} = \left\{\lambda + \epsilon z,~z \in [k]^d\right\},
    \end{equation}
    where we have written $[k]:=\{0,...,k-1\}$ for the set of integers from $0$ to $k-1$.
    
    Let us now investigate the set of sign patterns which the class $K$ can match on $X_\lambda$. To do this, we will introduce some notation. For a function $g\in K$, we write
    \begin{equation}
        \sign(g|_{X_\lambda}) \in \{\pm 1\}^{[k]^d},~\sign(g|_{X_\lambda})(z) = \sign(g(\lambda + \epsilon z))
    \end{equation}
    for the set of signs which $g$ takes at the (shifted) grid points $X_\lambda$. Here the vector $\sign(g|_{X_\lambda})$ is indexed by the coordinate $z\in [k]^d$ which specifies the location of a point in the shifted grid $X_\lambda$.
    
    We write
    \begin{equation}
        \sign(K|_{X_\lambda}) := \left\{\sign(g|_{X_\lambda}),~g\in K\right\}\subset \{\pm 1\}^{[k]^d}
    \end{equation}
    for the set of sign patterns attained by the class $K$ on $X_\lambda$. Observe that since $K$ is assumed to be translation invariant, the set $\sign(K|_{X_\lambda})$ is independent of the shift $\lambda$. To see this, let $\lambda,\mu\in [0,\epsilon)^d$ be two different shifts and let $g\in K$. By the translation invariance, we find that the function $g'$ defined by
    $$g'(x) = g(x + \lambda - \mu)$$
    is also in $K$. We easily calculate that
    \begin{equation}
        \sign(g|_{X_\lambda}) = \sign(g'|_{X_\mu}),
    \end{equation}
    which implies that $\sign(K|_{X_\lambda}) = \sign(K|_{X_\mu})$. In the following we simplify notation and write $\sign(K) \subset \{\pm 1\}^{[k]^d}$ for this set.

    Next, we will show that there exists a choice of signs $\alpha\in \{\pm 1\}^{[k]^d}$ which differs from every element of $\sign(K)$ in a constant fraction of its entries. To do this, it is convenient to use the notion of the Hamming distance between two sign patterns, which is defined as the number of indices in which they differ, i.e.
    \begin{equation}
        d_H(\alpha,\beta) := |\{z\in [k]^d:~\alpha(z) \neq \beta(z)\}|.
    \end{equation}
    We also use the notion of the Hamming ball of radius $m$ around a sign pattern $\alpha\in \{\pm 1\}^{[k]^d}$, which is defined to be the set of sign patterns which differ from $\alpha$ by at most $m$ entries, i.e.
    \begin{equation}
        B_H(\alpha,m) = \{\beta\in \{\pm 1\}^{[k]^d},~d_H(\alpha,\beta) \leq m\}.
    \end{equation}
    We note that Lemma \ref{binomial-sum-bound} implies the following estimate on the size of $B_H(\alpha,m)$ when $2m < k^d$:
    \begin{equation}
        |B_H(\alpha,m)| = \sum_{i=0}^m \binom{k^d}{i} \leq 2^{k^dH(m/k^d)}.
    \end{equation}
    
    Further, our assumption on the VC-dimension of $K$ combined with Lemmas \ref{sauer-shelah-lemma} and \ref{binomial-sum-bound} implies that
    \begin{equation}
        |\sign(K)| \leq 2^{k^dH(n/k^d)} \leq 2^{k^dH(c)} < 2^{k^d/2}
    \end{equation}
    from our choice of $c$. If we choose $m := \lfloor ck^d\rfloor \leq ck^d$, it follows that
    \begin{equation}
        \left|\bigcup_{\beta\in \sign(K)} B_H(\beta,m)\right| < 2^{k^d/2}2^{k^dH(m/k^d)} < 2^{k^d/2}2^{k^dH(c)} < 2^{k^d},
    \end{equation}
    so that there must exist an $\alpha\in \{\pm 1\}^{[k]^d}$ such that
    $$\alpha\notin \bigcup_{\beta\in \sign(K)} B_H(\beta,m),$$
    and hence
    \begin{equation}\label{hamming-distance-lower-bound}
        \inf_{\beta\in \sign(K)} d_H(\alpha, \beta) \geq m+1 \geq ck^d.
    \end{equation}

    Finally, we choose a compactly supported smooth positive bump function $\phi$ whose support is strictly contained in the unit cube $\Omega$ and consider the function
    \begin{equation}
        f(x) = \sum_{z\in [k]^d} \alpha(z) \phi(kx - z).
    \end{equation}
    Since the supports of the functions $\phi(kx - z)$ are all disjoint, we calculate
    \begin{equation}\label{besov-space-bound}
        \|f\|_{W^s(L_\infty(\Omega))} = \|\phi(kx - z)\|_{W^s(L_\infty(\Omega))} \leq k^s\|\phi\|_{W^s(L_\infty(\Omega))},
    \end{equation}
    and also
    \begin{equation}\label{besov-space-bound-2}
        \|f\|_{B^s_1(L_\infty(\Omega))} \leq 
        Ck^s\|\phi\|_{B^s_1(L_\infty(\Omega))},
    \end{equation}
    for an appropriate constant $C = C(d,s)$.
    Next, let $g\in K$ be arbitrary. We calculate
    \begin{equation}
    \begin{split}
        \int_{\Omega} |f(x) - g(x)|^pdx &= \int_{[0,\epsilon)^d}\sum_{z\in [k]^d} |f(\lambda + \epsilon z) - g(\lambda + \epsilon z)|^p d\lambda\\
        &= \int_{[0,\epsilon)^d} \sum_{z\in [k]^d} |\alpha(z)\phi(k\lambda) - g(\lambda + \epsilon z)|^pd\lambda.
    \end{split}
    \end{equation}
    From equation \eqref{hamming-distance-lower-bound} and the fact that $\sign(g|_{X_\lambda})\in \sign(K)$, we see that
    \begin{equation}
        |\{z\in [k]^d,~\alpha(z) \neq \sign(g(\lambda + \epsilon z))\}| \geq ck^d.
    \end{equation}
    Further, if $\alpha(z) \neq \sign(g(\lambda + \epsilon z))$, then we have the lower bound
    $$|\alpha(z)\phi(k\lambda) - g(\lambda + \epsilon z)| \geq \phi(k\lambda)$$
    since $\phi\geq 0$. This implies that for every $\lambda\in [0,\epsilon)^d$ we have the lower bound
    \begin{equation}
        \sum_{z\in [k]^d} |\alpha(z)\phi(k\lambda) - g(\lambda + \epsilon z)|^p \geq ck^d\phi(k\lambda)^p.
    \end{equation}
    We thus obtain
    \begin{equation}
        \int_{\Omega} |f(x) - g(x)|^pdx \geq ck^d\int_{[0,\epsilon)^d} \phi(k\lambda)^pd\lambda = c\int_{\Omega} \phi(x)^pdx,
    \end{equation}
    from which we deduce
    \begin{equation}
        \|f - g\|_{L_p(\Omega)} \geq c^{\frac{1}{p}}\|\phi\|_{L_p(\Omega)}.
    \end{equation}
Combining this with the bounds \eqref{besov-space-bound} and \eqref{besov-space-bound-2}, using that $k\leq C(d)n^{1/d}$, that $\phi$ is a fixed function, and that $g\in K$ was arbitrary, we get
\begin{equation}
    \inf_{g\in K} \|f - g\|_{L_p(\Omega)} \geq C(d,p)k^{-s}\|f\|_{W^s(L_\infty(\Omega))} \geq C(d,k)n^{-s/d}\max\left\{\|f\|_{W^s(L_\infty(\Omega))}, \|f\|_{B^s_1(L_\infty(\Omega))}\right\},
\end{equation}
as desired.
\end{proof}

We conclude this section by proving that Theorem \ref{sparse-approximation-theorem} is optimal up to a constant as long as the $\ell^1$-norm $M$ is not too large and not too small. Specifically, we have the following.
\begin{theorem}\label{sparse-approximation-lower-bound-theorem}
    Let $M,N\geq 1$ be integers and define
    \begin{equation}
        S_{N,M} = \{\textbf{\upshape  x}\in \mathbb{Z}^N,~\|\textbf{x}\|_{\ell^1} \leq M\}
    \end{equation}
    as in the proof of Theorem \ref{sparse-approximation-theorem}. Suppose that $W,L \geq 1$ are integers and that for any $\textbf{\upshape  x}\in S_{N,M}$ there exists an $f\in \Upsilon^{W,L}(\mathbb{R})$ such that $f(i) = \textbf{\upshape  x}_i$ for $i=1,..,N$. Then there exists a constant $C < \infty$ such that if $C\log(N) < M \leq N$, then
    \begin{equation}
        W^4L^2 \geq C^{-1}M(1+\log(N/M)),
    \end{equation}
    and if $N \leq M < \exp(N/C)$, then
    \begin{equation}
        W^4L^2 \geq C^{-1}N(1+\log(M/N)).
    \end{equation}
\end{theorem}
This result implies that if $\Upsilon^{W,L}(\mathbb{R})$ can match the values of any vector in $S_{N,M}$ for $M$ in the range $(C\log(N),\exp(N/C))$, then the number of parameters must be larger than a constant multiple of the upper bound proved in Theorem \ref{sparse-approximation-theorem}. Thus Theorem \ref{sparse-approximation-theorem} is sharp in this range. If $M < C\log(N)$ then piecewise linear functions with $O(M)$ pieces can fit $S_{N,M}$, and if $M > \exp(N/C)$ then piecewise linear functions with $O(N)$ pieces can fit $S_{N,M}$. This implies that Theorem \ref{sparse-approximation-theorem} is no longer sharp outside this range.
\begin{proof}
    Suppose first that $N/2 < M \leq 2N$, i.e. that $M$ is of the same order as $N$. For any subset $S\subset \{1,...,N/2\}$ it is easy to construct an $\textbf{x}\in S_{N,M}$ such that $\textbf{x}_i > 0$ iff $i\in S$. Thus the class $\Upsilon^{W,L}(\mathbb{R})$ must shatter a set of size at least $N/2$ and the VC-dimension bound \eqref{VC-dimension-bound} implies the result.
    
    In the case where $M << N$ or $M >> N$, the proof proceeds in a similar manner as the VC-dimension bounds from \cite{goldberg1993bounding,bartlett2019nearly} although the VC-dimension cannot directly be used. 
    
    We begin with the case where $M \leq N/2$. We will bound the total number of sign patterns that $\Upsilon^{W,L}(\mathbb{R})$ can match on the input set $X = \{1,...,N\}$. For $i=1,...,L$, let $\epsilon_i\in \{0,1\}^W$ be a sign pattern. Given an input $x\in X$ and a neural network with parameters $\textbf{W}_i$ and $b_i$, consider the signs of the following quantities
    \begin{equation}\label{sign-collection-small-M}
    \begin{split}
        (A_{\textbf{W}_0,b_0}(x))_j,&~j=1,...,W\\
        (A_{\textbf{W}_1,b_1} \circ \epsilon_1 \circ A_{\textbf{W}_0,b_0}(x))_j,&~j=1,...,W\\
        (A_{\textbf{W}_2,b_2} \circ \epsilon_2 \circ A_{\textbf{W}_1,b_1} \circ \epsilon_1 \circ A_{\textbf{W}_0,b_0}(x))_j,&~j=1,...,W\\
        &\vdots\\
        (A_{\textbf{W}_{L-1},b_{L-1}} \circ \epsilon_{L-1} \circ \cdots \circ \epsilon_2 \circ A_{\textbf{W}_1,b_1} \circ \epsilon_1 \circ A_{\textbf{W}_0,b_0}(x))_j,&~j=1,..,W\\
        A_{\textbf{W}_L,b_L} \circ \epsilon_L \circ A_{\textbf{W}_{L-1},b_{L-1}} \circ \epsilon_{L-1} \circ \cdots \circ \epsilon_2 \circ A_{\textbf{W}_1,b_1} \circ \epsilon_1 \circ A_{\textbf{W}_0,b_0}(x).~~~&
    \end{split}
    \end{equation}
    Here $\epsilon_i$ represents pointwise multiplication by the sign vector $\epsilon_i$. For any input $x\in \mathbb{R}$ the definition of the ReLU activation function implies that if we recursively set
    \begin{equation}\label{recursive-epsilon-equation}
        \epsilon_i = \sign(A_{\textbf{W}_{i-1},b_{i-1}} \circ \epsilon_{i-1} \circ \cdots \circ \epsilon_2 \circ A_{\textbf{W}_1,b_1} \circ \epsilon_1 \circ A_{\textbf{W}_0,b_0}(x)),
    \end{equation}
    then we will have
    \begin{equation}
        A_{\textbf{W}_L,b_L} \circ \epsilon_L \circ  \cdots \circ \epsilon_2 \circ A_{\textbf{W}_1,b_1} \circ \epsilon_1 \circ A_{\textbf{W}_0,b_0}(x) = A_{\textbf{W}_L,b_L} \circ \sigma \circ \cdots \circ \sigma \circ A_{\textbf{W}_1,b_1} \circ \sigma \circ A_{\textbf{W}_0,b_0}(x).
    \end{equation}
    This implies that the signs of the quantities in \eqref{sign-collection-small-M} ranging over all sign vectors $\epsilon_1,...,\epsilon_L\in \{0,1\}^W$ uniquely determine the value of the network at $x$. Thus the number of sign patterns achieved on the set $X$ is bounded by the number of sign patterns achieved in \eqref{sign-collection-small-M} as $x$ ranges over the input set $X$, the $\epsilon_i$ range over the sign vectors $\{0,1\}^W$, and the parameters $\textbf{W}_i$, $b_i$ range of the set of all real numbers. As the $\epsilon_i$ range over the sign vectors $\{0,1\}^W$ and $x$ ranges over $X$, the quantities in \eqref{sign-collection-small-M} range over $N(WL+1)2^{WL}$ polynomials in the $P \leq CW^2L$ parameter variables $\textbf{W}_i$, $b_i$ of degree at most $L$. We can thus use Warren's Theorem (\cite{warren1968lower}, Theorem 3) to bound the total number of sign patterns by
    \begin{equation}
        \left(\frac{4eLN(WL+1)2^{WL}}{P}\right)^P\leq (4eLN(WL+1)2^{WL})^{CW^2L}.
    \end{equation}
    Suppose that $\Upsilon^{W,L}(\mathbb{R})$ can match the values of any element in $S_{N,M}$. Since the set $S_{N,M}$ contains the indicator function of every subset of $\{1,...,N\}$ of size $M$, we get that
    \begin{equation}
        \binom{N}{M} \leq (4eLN(WL+1)2^{WL})^{CW^2L}.
    \end{equation}
    Taking logarithms, we get
    \begin{equation}
    \begin{split}
        M\log(N/M) &\leq CW^3L^2 + CW^2L\log(N) + CW^2L\log(4eL(WL+1))\\
        &\leq CW^3L^2 + CW^2L\log(N) \leq CW^4L^2 + CW^2L\log(N).
    \end{split}
    \end{equation}
    Since $M \leq N/2$, we conclude that
    \begin{equation}\label{sign-pattern-bound-3725}
        M(1+\log(N/M)) \leq CM\log(N/M) \leq C\max\{W^4L^2, W^2L\log(N)\}.
    \end{equation}
    In the next few equations, let $C$ denote the constant in \eqref{sign-pattern-bound-3725}. Suppose that $W^4L^2 < C^{-1}M(1+\log(N/M))$. Then equation \eqref{sign-pattern-bound-3725} implies that
    \begin{equation}
        W^2L \geq M\frac{1+\log(N/M)}{C\log(N)}.
    \end{equation}
    But this would mean that
    \begin{equation}
        M\frac{1+\log(N/M)}{C\log(N)} \leq W^2L = \sqrt{W^4L^2} < \sqrt{C^{-1}M(1+\log(N/M))}.
    \end{equation}
    Rearranging this, we get the inequality
    \begin{equation}
        \sqrt{M} < \frac{\sqrt{C}\log(N)}{\sqrt{1+\log(N/M)}},
    \end{equation}
    from which we deduce that $M \leq C\log(N)$ for a (potentially larger) new constant $C$.

    Next, we consider the case where $M > 2N$. In this case we consider the following modification of \eqref{sign-collection-small-M}
    \begin{equation}\label{sign-collection-large-M}
    \begin{split}
        (A_{\textbf{W}_0,b_0}(x))_j,&~j=1,...,W\\
        (A_{\textbf{W}_1,b_1} \circ \epsilon_1 \circ A_{\textbf{W}_0,b_0}(x))_j,&~j=1,...,W\\
        (A_{\textbf{W}_2,b_2} \circ \epsilon_2 \circ A_{\textbf{W}_1,b_1} \circ \epsilon_1 \circ A_{\textbf{W}_0,b_0}(x))_j,&~j=1,...,W\\
        &\vdots\\
        (A_{\textbf{W}_{L-1},b_{L-1}} \circ \epsilon_{L-1} \circ \cdots \circ \epsilon_2 \circ A_{\textbf{W}_1,b_1} \circ \epsilon_1 \circ A_{\textbf{W}_0,b_0}(x))_j,&~j=1,..,W\\
        A_{\textbf{W}_L,b_L} \circ \epsilon_L \circ A_{\textbf{W}_{L-1},b_{L-1}} \circ \epsilon_{L-1} \circ \cdots \circ \epsilon_2 \circ A_{\textbf{W}_1,b_1} \circ \epsilon_1 \circ A_{\textbf{W}_0,b_0}(x) - k,&~k=0,...\lfloor M/N\rfloor-1.
    \end{split}
    \end{equation}
    The number of sign patterns that can be obtained as $x$ ranges over $X$, the $\epsilon_i$ range over $\{0,1\}^W$, and the parameters range over the set of real numbers is bounded (using Warren's Theorem \cite{warren1968lower}) by
    \begin{equation}
        \left(\frac{4eLN(WL+M)2^{WL}}{P}\right)^P\leq (4eLN(WL+M)2^{WL})^{CW^2L}.
    \end{equation}
    The only difference to the previous bound is that for each choice of $x\in X$ and each choice of signs $\epsilon_i\in \{0,1\}^W$, the number of equations in \eqref{sign-collection-large-M} is $WL + \lfloor M/N\rfloor \leq WL + M$.
    
    However, the set $S_{N,M}$ contains all $(\lfloor M/N\rfloor+1)^{N-1}$ vectors whose first $N-1$ coordinates are arbitrary integers in $\{0,1,...,\lfloor M/N\rfloor\}$ and whose last coordinate is chosen to make the $\ell^1$-norm equal to $M$. Thus, setting $\epsilon_i$ recursively according to \eqref{recursive-epsilon-equation}, we see that if every vector in $S_{N,M}$ can be represented by an element of $\Upsilon^{W,L}(\mathbb{R})$, then
    \begin{equation}
        (\lfloor M/N\rfloor+1)^{N-1} \leq (4eLN(WL+M)2^{WL})^{CW^2L}.
    \end{equation}
    Taking logarithms and calculating in a similar manner as before, we get
    \begin{equation}
        N(1+\log(M/N)) \leq CW^4L^2 + CW^2L\log(M).
    \end{equation}
    As before we now deduce that if $W^4L^2 < C^{-1}N(1+\log(M/N))$, then $N \leq C\log(M)$.
\end{proof}

\section{Acknowledgements}
We would like to thank Ron DeVore for suggesting this problem, and Andrea Bonito, Geurgana Petrova, Zuowei Shen, George Karniadakis, Jinchao Xu, and Qihang Zhou for helpful comments while preparing this manuscript. We would also like to thank the anonymous reviewers for their careful reading and helpful comments. This work was supported by the National Science Foundation (DMS-2111387 and CCF-2205004) as well as a MURI ONR grant N00014-20-1-2787.

\bibliographystyle{amsplain}
\bibliography{refs}

\providecommand{\bysame}{\leavevmode\hbox to3em{\hrulefill}\thinspace}
\providecommand{\MR}{\relax\ifhmode\unskip\space\fi MR }
\providecommand{\MRhref}[2]{%
  \href{http://www.ams.org/mathscinet-getitem?mr=#1}{#2}
}
\providecommand{\href}[2]{#2}
\begin{thebibliography}{10}

\bibitem{achour2022general}
El~Mehdi Achour, Armand Foucault, S{\'e}bastien Gerchinovitz, and
  Fran{\c{c}}ois Malgouyres, \emph{A general approximation lower bound in
  ${L}^p$ norm, with applications to feed-forward neural networks}, arXiv
  preprint arXiv:2206.04360 (2022).

\bibitem{ajtai19830}
Mikl{\'o}s Ajtai, J{\'a}nos Koml{\'o}s, and Endre Szemer{\'e}di, \emph{An {$O(n
  \log n)$} sorting network}, Proceedings of the Fifteenth Annual ACM Symposium
  on Theory of computing, 1983, pp.~1--9.

\bibitem{arora2018understanding}
Raman Arora, Amitabh Basu, Poorya Mianjy, and Anirbit Mukherjee,
  \emph{Understanding deep neural networks with rectified linear units},
  International Conference on Learning Representations, 2018.

\bibitem{bach2017breaking}
Francis Bach, \emph{Breaking the curse of dimensionality with convex neural
  networks}, The Journal of Machine Learning Research \textbf{18} (2017),
  no.~1, 629--681.

\bibitem{bartlett1998almost}
Peter Bartlett, Vitaly Maiorov, and Ron Meir, \emph{Almost linear {VC}
  dimension bounds for piecewise polynomial networks}, Advances in Neural
  Information Processing Systems \textbf{11} (1998).

\bibitem{bartlett2019nearly}
Peter~L Bartlett, Nick Harvey, Christopher Liaw, and Abbas Mehrabian,
  \emph{Nearly-tight {VC}-dimension and pseudodimension bounds for piecewise
  linear neural networks}, The Journal of Machine Learning Research \textbf{20}
  (2019), no.~1, 2285--2301.

\bibitem{batcher1968sorting}
Kenneth~E Batcher, \emph{Sorting networks and their applications}, Proceedings
  of the April 30--May 2, 1968, Spring Joint Computer Conference, 1968,
  pp.~307--314.

\bibitem{birman1967piecewise}
Mikhail~Shlemovich Birman and Mikhail~Zakharovich Solomyak,
  \emph{Piecewise-polynomial approximations of functions of the classes
  {$W_p^\alpha$}}, Matematicheskii Sbornik \textbf{115} (1967), no.~3,
  331--355.

\bibitem{bramble1970estimation}
James~H Bramble and SR~Hilbert, \emph{Estimation of linear functionals on
  {S}obolev spaces with application to {F}ourier transforms and spline
  interpolation}, SIAM Journal on Numerical Analysis \textbf{7} (1970), no.~1,
  112--124.

\bibitem{bramble1990parallel}
James~H Bramble, Joseph~E Pasciak, and Jinchao Xu, \emph{Parallel multilevel
  preconditioners}, Mathematics of Computation \textbf{55} (1990), no.~191,
  1--22.

\bibitem{chambolle1998nonlinear}
Antonin Chambolle, Ronald~A DeVore, Nam-Yong Lee, and Bradley~J Lucier,
  \emph{Nonlinear wavelet image processing: variational problems, compression,
  and noise removal through wavelet shrinkage}, IEEE Transactions on Image
  Processing \textbf{7} (1998), no.~3, 319--335.

\bibitem{daubechies1992ten}
Ingrid Daubechies, \emph{Ten {L}ectures on {W}avelets}, SIAM, 1992.

\bibitem{daubechies2022neural}
Ingrid Daubechies, Ronald DeVore, Nadav Dym, Shira Faigenbaum-Golovin, Shahar~Z
  Kovalsky, Kung-Chin Lin, Josiah Park, Guergana Petrova, and Barak Sober,
  \emph{Neural network approximation of refinable functions}, IEEE Transactions
  on Information Theory (2022).

\bibitem{daubechies2022nonlinear}
Ingrid Daubechies, Ronald DeVore, Simon Foucart, Boris Hanin, and Guergana
  Petrova, \emph{Nonlinear approximation and (deep) {ReLU} networks},
  Constructive Approximation \textbf{55} (2022), no.~1, 127--172.

\bibitem{demengel2012functional}
Fran{\c{c}}oise Demengel, Gilbert Demengel, and Reinie Ern{\'e},
  \emph{{Functional Spaces for the Theory of Elliptic Partial Differential
  Equations}}, Springer, 2012.

\bibitem{devore2021neural}
Ronald DeVore, Boris Hanin, and Guergana Petrova, \emph{Neural network
  approximation}, Acta Numerica \textbf{30} (2021), 327--444.

\bibitem{devore1998nonlinear}
Ronald~A DeVore, \emph{Nonlinear approximation}, Acta Numerica \textbf{7}
  (1998), 51--150.

\bibitem{devore1992image}
Ronald~A DeVore, Bj{\"o}rn Jawerth, and Bradley~J Lucier, \emph{Image
  compression through wavelet transform coding}, IEEE Transactions on
  Information Theory \textbf{38} (1992), no.~2, 719--746.

\bibitem{devore1993constructive}
Ronald~A DeVore and George~G Lorentz, \emph{Constructive {A}pproximation}, vol.
  303, Springer Science \& Business Media, 1993.

\bibitem{devore1988interpolation}
Ronald~A DeVore and Vasil~A Popov, \emph{Interpolation of besov spaces},
  Transactions of the American Mathematical Society \textbf{305} (1988), no.~1,
  397--414.

\bibitem{devore1984maximal}
Ronald~A DeVore and Robert~C Sharpley, \emph{Maximal functions measuring
  smoothness}, vol. 293, American Mathematical Soc., 1984.

\bibitem{devore1993besov}
\bysame, \emph{Besov spaces on domains in $\mathbb{R}^d$}, Transactions of the
  American Mathematical Society \textbf{335} (1993), no.~2, 843--864.

\bibitem{di2012hitchhikers}
Eleonora Di~Nezza, Giampiero Palatucci, and Enrico Valdinoci, \emph{Hitchhikers
  guide to the fractional {S}obolev spaces}, Bulletin des Sciences
  Math{\'e}matiques \textbf{136} (2012), no.~5, 521--573.

\bibitem{donoho1995adapting}
David~L Donoho and Iain~M Johnstone, \emph{Adapting to unknown smoothness via
  wavelet shrinkage}, Journal of the American Statistical Association
  \textbf{90} (1995), no.~432, 1200--1224.

\bibitem{donoho1998minimax}
\bysame, \emph{Minimax estimation via wavelet shrinkage}, The Annals of
  Statistics \textbf{26} (1998), no.~3, 879--921.

\bibitem{donoho1998data}
David~L. Donoho, Martin Vetterli, Ronald~A. DeVore, and Ingrid Daubechies,
  \emph{Data compression and harmonic analysis}, IEEE Transactions on
  Information Theory \textbf{44} (1998), no.~6, 2435--2476.

\bibitem{evans2010partial}
Lawrence~C Evans, \emph{Partial {D}ifferential {E}quations}, vol.~19, American
  Mathematical Soc., 2010.

\bibitem{goldberg1993bounding}
Paul Goldberg and Mark Jerrum, \emph{Bounding the {V}apnik-{C}hervonenkis
  dimension of concept classes parameterized by real numbers}, Proceedings of
  the Sixth Annual Conference on Computational Learning Theory, 1993,
  pp.~361--369.

\bibitem{goodfellow2016deep}
Ian Goodfellow, Yoshua Bengio, and Aaron Courville, \emph{Deep {L}earning}, MIT
  press, 2016.

\bibitem{guhring2020error}
Ingo G{\"u}hring, Gitta Kutyniok, and Philipp Petersen, \emph{Error bounds for
  approximations with deep {ReLU} neural networks in {$W^{s,p}$} norms},
  Analysis and Applications \textbf{18} (2020), no.~05, 803--859.

\bibitem{han2018solving}
Jiequn Han, Arnulf Jentzen, and Weinan E, \emph{Solving high-dimensional
  partial differential equations using deep learning}, Proceedings of the
  National Academy of Sciences \textbf{115} (2018), no.~34, 8505--8510.

\bibitem{hanin2019universal}
Boris Hanin, \emph{Universal function approximation by deep neural nets with
  bounded width and {ReLU} activations}, Mathematics \textbf{7} (2019), no.~10,
  992.

\bibitem{hanin2019complexity}
Boris Hanin and David Rolnick, \emph{Complexity of linear regions in deep
  networks}, International Conference on Machine Learning, PMLR, 2019,
  pp.~2596--2604.

\bibitem{he2020relu}
Juncai He, Lin Li, Jinchao Xu, and Chunyue Zheng, \emph{{ReLU} deep neural
  networks and linear finite elements}, Journal of Computational Mathematics
  \textbf{38} (2020), no.~3, 502--527.

\bibitem{klusowski2018approximation}
Jason~M Klusowski and Andrew~R Barron, \emph{Approximation by combinations of
  {ReLU} and squared {ReLU} ridge functions with $\ell^1$ and $\ell^0$
  controls}, IEEE Transactions on Information Theory \textbf{64} (2018),
  no.~12, 7649--7656.

\bibitem{kufner1977function}
Alois Kufner, Oldrich John, and Svatopluk Fucik, \emph{Function {S}paces},
  vol.~3, Springer Science \& Business Media, 1977.

\bibitem{lecun2015deep}
Yann LeCun, Yoshua Bengio, and Geoffrey Hinton, \emph{Deep learning}, Nature
  \textbf{521} (2015), no.~7553, 436--444.

\bibitem{littlewood1931theorems}
John~E Littlewood and Raymond~EAC Paley, \emph{Theorems on {F}ourier series and
  power series}, Journal of the London Mathematical Society \textbf{1} (1931),
  no.~3, 230--233.

\bibitem{lorentz1996constructive}
George~G Lorentz, Manfred~v Golitschek, and Yuly Makovoz, \emph{Constructive
  {A}pproximation: {A}dvanced {P}roblems}, vol. 304, Springer, 1996.

\bibitem{lu2021deep}
Jianfeng Lu, Zuowei Shen, Haizhao Yang, and Shijun Zhang, \emph{Deep network
  approximation for smooth functions}, SIAM Journal on Mathematical Analysis
  \textbf{53} (2021), no.~5, 5465--5506.

\bibitem{mallat1999wavelet}
St{\'e}phane Mallat, \emph{A {W}avelet {T}our of {S}ignal {P}rocessing},
  Elsevier, 1999.

\bibitem{nair2010rectified}
Vinod Nair and Geoffrey~E Hinton, \emph{Rectified linear units improve
  restricted boltzmann machines}, International Conference on Machine Learning,
  2010.

\bibitem{paterson1990improved}
Michael~S Paterson, \emph{Improved sorting networks with {$O(\log N)$} depth},
  Algorithmica \textbf{5} (1990), no.~1, 75--92.

\bibitem{petersen2018optimal}
Philipp Petersen and Felix Voigtlaender, \emph{Optimal approximation of
  piecewise smooth functions using deep {ReLU} neural networks}, Neural
  Networks \textbf{108} (2018), 296--330.

\bibitem{petrushev1988direct}
Pencho~P Petrushev, \emph{Direct and converse theorems for spline and rational
  approximation and {B}esov spaces}, Function Spaces and Applications:
  Proceedings of the US-Swedish Seminar held in Lund, Sweden, June 15--21,
  1986, Springer, 1988, pp.~363--377.

\bibitem{raissi2019physics}
Maziar Raissi, Paris Perdikaris, and George~E Karniadakis,
  \emph{Physics-informed neural networks: A deep learning framework for solving
  forward and inverse problems involving nonlinear partial differential
  equations}, Journal of Computational physics \textbf{378} (2019), 686--707.

\bibitem{sauer1972density}
Norbert Sauer, \emph{On the density of families of sets}, Journal of
  Combinatorial Theory, Series A \textbf{13} (1972), no.~1, 145--147.

\bibitem{serra2018bounding}
Thiago Serra, Christian Tjandraatmadja, and Srikumar Ramalingam, \emph{Bounding
  and counting linear regions of deep neural networks}, International
  Conference on Machine Learning, PMLR, 2018, pp.~4558--4566.

\bibitem{shelah1972combinatorial}
Saharon Shelah, \emph{A combinatorial problem; stability and order for models
  and theories in infinitary languages}, Pacific Journal of Mathematics
  \textbf{41} (1972), no.~1, 247--261.

\bibitem{shen2022optimal}
Zuowei Shen, Haizhao Yang, and Shijun Zhang, \emph{Optimal approximation rate
  of {ReLU} networks in terms of width and depth}, Journal de Math{\'e}matiques
  Pures et Appliqu{\'e}es \textbf{157} (2022), 101--135.

\bibitem{shijun2021deep}
Zhang Shijun, \emph{Deep neural network approximation via function
  compositions}, Ph.D. thesis, National University of Singapore (Singapore),
  2021.

\bibitem{siegel2020approximation}
Jonathan~W Siegel and Jinchao Xu, \emph{Approximation rates for neural networks
  with general activation functions}, Neural Networks \textbf{128} (2020),
  313--321.

\bibitem{siegel2022high}
\bysame, \emph{High-order approximation rates for shallow neural networks with
  cosine and {ReLU$^k$} activation functions}, Applied and Computational
  Harmonic Analysis \textbf{58} (2022), 1--26.

\bibitem{siegel2022sharp}
\bysame, \emph{Sharp bounds on the approximation rates, metric entropy, and
  n-widths of shallow neural networks}, Foundations of Computational
  Mathematics (2022), 1--57.

\bibitem{telgarsky2016benefits}
Matus Telgarsky, \emph{Benefits of depth in neural networks}, Conference on
  Learning Theory, PMLR, 2016, pp.~1517--1539.

\bibitem{vapnik2015uniform}
Vladimir~N Vapnik and A~Ya Chervonenkis, \emph{On the uniform convergence of
  relative frequencies of events to their probabilities}, Measures of
  Complexity, Springer, 2015, pp.~11--30.

\bibitem{wang2005generalization}
Shuning Wang and Xusheng Sun, \emph{Generalization of hinging hyperplanes},
  IEEE Transactions on Information Theory \textbf{51} (2005), no.~12,
  4425--4431.

\bibitem{warren1968lower}
Hugh~E Warren, \emph{Lower bounds for approximation by nonlinear manifolds},
  Transactions of the American Mathematical Society \textbf{133} (1968), no.~1,
  167--178.

\bibitem{whitney1934analytic}
Hassler Whitney, \emph{Analytic extensions of differentiable functions defined
  in closed sets}, Transactions of the American Mathematical Society
  \textbf{36} (1934), no.~1, 63--89.

\bibitem{yarotsky2017error}
Dmitry Yarotsky, \emph{Error bounds for approximations with deep {ReLU}
  networks}, Neural Networks \textbf{94} (2017), 103--114.

\bibitem{yarotsky2018optimal}
\bysame, \emph{Optimal approximation of continuous functions by very deep
  {ReLU} networks}, Conference on Learning Theory, PMLR, 2018, pp.~639--649.

\bibitem{yarotsky2020phase}
Dmitry Yarotsky and Anton Zhevnerchuk, \emph{The phase diagram of approximation
  rates for deep neural networks}, Advances in Neural Information Processing
  Pystems \textbf{33} (2020), 13005--13015.

\bibitem{yuan2010morrey}
Wen Yuan, Winfried Sickel, and Dachun Yang, \emph{Morrey and {Campanato Meet
  Besov, Lizorkin and Triebel}}, Springer, 2010.

\end{thebibliography}

\appendix
\section{Elementary Constructions}\label{elementary-appendix}
Here we give the constructions of sums and of piecewise linear continuous functions using deep ReLU networks references in Section \ref{basic-network-constructions-section}.
\begin{proof}[Proof of Proposition \ref{summing-networks}]
    	We will show by induction on $j$ that
	\begin{equation}
	\begin{pmatrix}
            x\\
            0
        \end{pmatrix}\rightarrow
        \begin{pmatrix}
            x\\
            \sum_{i=1}^j f_i(x)
        \end{pmatrix}\in \Upsilon^{W+2d+2k,L}(\mathbb{R}^{d+k},\mathbb{R}^{d+k})
	\end{equation}
	for $L=\sum_{i=1}^j L_i$. The base case $j=0$ is trivial since the identity map is affine. Suppose we have shown this for $j-1$, i.e.
	\begin{equation}
	\begin{pmatrix}
            x\\
            0
        \end{pmatrix}\rightarrow
        \begin{pmatrix}
            x\\
            \sum_{i=1}^{j-1} f_i(x)
        \end{pmatrix}\in \Upsilon^{W+2d+2k,L}(\mathbb{R}^{d+k},\mathbb{R}^{d+k}),
	\end{equation}
	where $L = \sum_{i=1}^{j-1} L_i$. Compose this map with an affine map which duplicates the first entry to get
	\begin{equation}
	\begin{pmatrix}
            x\\
            0
        \end{pmatrix}\rightarrow
        \begin{pmatrix}
            x\\
            x\\
            \sum_{i=1}^{j-1} f_i(x)
        \end{pmatrix}\in \Upsilon^{W+2d+2k,L}(\mathbb{R}^{d+k},\mathbb{R}^{2d+k}).
	\end{equation}
	Now, we use Lemma \ref{select-coordinates-lemma} to apply $f_j$ to the middle entry. This gives
	\begin{equation}
	\begin{pmatrix}
            x\\
            0
        \end{pmatrix}\rightarrow
        \begin{pmatrix}
            x\\
            f_j(x)\\
            \sum_{i=1}^{j-1} f_i(x)
        \end{pmatrix}\in \Upsilon^{W+2d+2k,L+L_j}(\mathbb{R}^{d+k},\mathbb{R}^{d+2k}).
	\end{equation}
	We finally compose with the affine map
	\begin{equation}
		\begin{pmatrix}
            x\\
            y\\
            z
        \end{pmatrix}\rightarrow 
        		\begin{pmatrix}
            x\\
            y+z
        \end{pmatrix}\in \Upsilon^{0}(\mathbb{R}^{d+2k},\mathbb{R}^{d+k}),
	\end{equation}
	and apply Lemma \ref{composition-lemma} to obtain
	\begin{equation}
		\begin{pmatrix}
            x\\
            0
        \end{pmatrix}\rightarrow
        \begin{pmatrix}
            x\\
            \sum_{i=1}^j f_i(x)
        \end{pmatrix}\in \Upsilon^{W+2d+2k,L+L_j}(\mathbb{R}^{d+k},\mathbb{R}^{d+k}),
	\end{equation}
	which completes the inductive step. 
	
	Applying the induction up to $j=n$, we have that
	\begin{equation}
	x\rightarrow \begin{pmatrix}
            x\\
            0
        \end{pmatrix}\rightarrow \begin{pmatrix}
            x\\
            \sum_{i=1}^n f_i(x)
        \end{pmatrix}\rightarrow  \sum_{i=1}^n f_i(x)\in \Upsilon^{W+2k+2,L}(\mathbb{R}^d,\mathbb{R}^k),
	\end{equation}
	where $L = \sum_{i=1}^n L_i$, since the first and last maps above are affine (applying Lemma \ref{composition-lemma}).
\end{proof}

\begin{proof}[Proof of Proposition \ref{piecewise-linear-function-proposition}]
    First observe that any piecewise linear function $f$ with $k$ pieces can be written as
    \begin{equation}\label{piecewise-linear}
        f(x) = a_0x + c + \sum_{i=1}^{k-1}a_i\sigma(x - b_i)
    \end{equation}
    for appropriate weights $a_0,...,a_{k-1}$ and $b_1,...,b_{k-1}$. Specifically, the $b_i$ are simply equal to the breakpoints points at which the derivative of $f$ is discontinuous, while the $a_i$ give the jump in derivative at those points. $a_0$ is set equal to the derivative in the left-most component and $c$ is set to match the value at $0$.
    
    Now we apply Proposition \ref{summing-networks} to the sum \eqref{piecewise-linear} to get the desired result, since we easily see that
    $$
    x\rightarrow a_0x + c\in \Upsilon^0(\mathbb{R})
    $$
    and
    $$
    x\rightarrow a_i\sigma(x-b_i)\in \Upsilon^{1,1}(\mathbb{R}).
    $$
\end{proof}
\begin{proof}[Proof of Lemma \ref{max-min-networks-lemma}]
	We observe the basic formulas:
	\begin{equation}
		\max(x,y) = x + \sigma(y-x),~\min(x,y) = x - \sigma(x-y).
	\end{equation}
	We begin with the affine map
	\begin{equation}
		\begin{pmatrix}
            x\\
            y
        \end{pmatrix}\rightarrow \begin{pmatrix}
            x\\
            y-x\\
            x-y
        \end{pmatrix}\in \Upsilon^0(\mathbb{R}^2,\mathbb{R}^3).
	\end{equation}
	Next, we use the fact that $\sigma\in \Upsilon^{1,1}(\mathbb{R})$ and Lemmas \ref{composition-lemma} and \ref{select-coordinates-lemma} to apply $\sigma$ to the last two coordinates. We get
	\begin{equation}
		\begin{pmatrix}
            x\\
            y
        \end{pmatrix}\rightarrow \begin{pmatrix}
            x\\
            \sigma(y-x)\\
            \sigma(x-y)
        \end{pmatrix}\in \Upsilon^{4,1}(\mathbb{R}^2,\mathbb{R}^3).
	\end{equation}
	Finally, we use Lemma \ref{composition-lemma} to compose with the affine map
	\begin{equation}
	\begin{pmatrix}
            x\\
            y\\
            z
        \end{pmatrix}\rightarrow \begin{pmatrix}
            x+y\\
            x-z
        \end{pmatrix}.
	\end{equation}
\end{proof}
\section{Product Network Construction}\label{product-network-appendix}
\begin{proof}[Proof of Proposition \ref{produt-network-proposition}]
    Observe that the piecewise linear hat function
    \begin{equation}
        f(x) = \begin{cases}
            2x & x\leq 1/2\\
            2(1 - x) & x > 1/2
        \end{cases}
    \end{equation}
    satisfies $f\in \Upsilon^{5,1}(\mathbb{R})$ by Proposition \ref{piecewise-linear-function-proposition}. On the interval $[0,1]$, $f$ composed with itself $n$ times is the sawtooth function
    \begin{equation}
        f^{\circ n}(x) := (f \circ \cdots \circ f)(x) = f(2^{n-1}x - \lfloor 2^{n-1}x\rfloor),
    \end{equation}
    and one can calculate that (see \cite{yarotsky2017error})
    \begin{equation}\label{x-sq-formula}
        x^2 = x - \sum_{n=1}^\infty 4^{-n}f^{\circ n}(x)
    \end{equation}
    for $x\in [0,1]$. 
    
    Using this, we construct a network $g_k\in \Upsilon^{7,k}(\mathbb{R})$ such that
    \begin{equation}\label{x-sq-error}
        \sup_{x\in [0,1]}|x^2 - g_k(x)| \leq 4^{-k}.
    \end{equation}
    To do this, we first apply the affine map which duplicates the input
    \begin{equation}\label{eq-813}
        x\rightarrow \begin{pmatrix}
            x\\
            x
        \end{pmatrix}\in \Upsilon^0(\mathbb{R},\mathbb{R}^2).
    \end{equation}
    Next, we show by induction on $k$ that the map
    \begin{equation}\label{inductive-step-820}
        x\rightarrow \begin{pmatrix}
            x - \sum_{n=1}^k 4^{-n}f^{\circ n}(x)\\
            f^{\circ k}(x)
        \end{pmatrix}\in \Upsilon^{7,k}(\mathbb{R},\mathbb{R}^2).
    \end{equation}
    The base case $k=0$ is simply \eqref{eq-813}. 
    
    For the inductive step suppose that \eqref{inductive-step-820} holds for $k\geq 0$. We use Lemma \ref{select-coordinates-lemma} to apply $f\in \Upsilon^{5,1}(\mathbb{R})$ to the second coordinate, showing that
    \begin{equation}
        \begin{pmatrix}
            x\\
            y
        \end{pmatrix}\rightarrow \begin{pmatrix}
            x\\
            f(y)
        \end{pmatrix}\in \Upsilon^{7,1}(\mathbb{R}^2,\mathbb{R}^2).
    \end{equation}
    Using the inductive assumption and Lemma \ref{composition-lemma}, we compose this with the map in \eqref{inductive-step-820} to get
    \begin{equation}
        x\rightarrow \begin{pmatrix}
            x - \sum_{n=1}^k 4^{-n}f^{\circ n}(x)\\
            f^{\circ (k+1)}(x)
        \end{pmatrix}\in \Upsilon^{7,k+1}(\mathbb{R},\mathbb{R}^2).
    \end{equation}
    We again use Lemma \ref{composition-lemma} and compose with the affine map
    \begin{equation}
        \begin{pmatrix}
            x\\
            y
        \end{pmatrix}\rightarrow \begin{pmatrix}
            x+y\\
            y
        \end{pmatrix}\in \Upsilon^{0}(\mathbb{R}^2,\mathbb{R}^2)
    \end{equation}
    to complete the inductive step.
    
    To construct $g_k$ we then simply compose the map in \eqref{inductive-step-820} with the affine map
    \begin{equation}
        \begin{pmatrix}
            x\\
            y
        \end{pmatrix}\rightarrow x\in \Upsilon^{0}(\mathbb{R}^2,\mathbb{R})
    \end{equation}
    which forgets the second coordinate. Then for $x\in [0,1]$ we have
    \begin{equation}
        g_k(x) = x - \sum_{n=1}^k 4^{-n}f^{\circ n}(x)
    \end{equation}
    and by \eqref{x-sq-formula} we get the bound \eqref{x-sq-error}. This gives us a network which approximates $x^2$ on the interval $[0,1]$. 
    
    In order to obtain a network which approximates $x^2$ on $[-1,1]$ we observe that if $x\in [-1,1]$, then $\sigma(x),\sigma(-x)\in [0,1]$, and
    \begin{equation}
        x^2 = \sigma(x)^2 + \sigma(-x)^2.
    \end{equation}
    We begin with the single layer network
    \begin{equation}\label{expanding-eq-877}
        x\rightarrow \begin{pmatrix}
            \sigma(x)\\
            \sigma(-x)
        \end{pmatrix}\in \Upsilon^{2,1}(\mathbb{R},\mathbb{R}^2).
    \end{equation}
    Further, applying Lemma \ref{select-coordinates-lemma}, we see that
    \begin{equation}
        \begin{pmatrix}
            x\\
            y
        \end{pmatrix}\rightarrow \begin{pmatrix}
            g_k(x)\\
            y
        \end{pmatrix} \in \Upsilon^{9,k}(\mathbb{R}^2,\mathbb{R}^2)
    \end{equation}
    and also
    \begin{equation}
        \begin{pmatrix}
            x\\
            y
        \end{pmatrix}\rightarrow \begin{pmatrix}
            x\\
            g_k(y)
        \end{pmatrix} \in \Upsilon^{9,k}(\mathbb{R}^2,\mathbb{R}^2)
    \end{equation}
    Finally, composing all of these and then applying the affine summation map
    \begin{equation}
        \begin{pmatrix}
            x\\
            y
        \end{pmatrix}\rightarrow x + y\in \Upsilon^0(\mathbb{R}^2),
    \end{equation}
    we get, using Lemma \ref{composition-lemma} (note that we can expand the width of the network in \eqref{expanding-eq-877}), a function $h_k\in \Upsilon^{9,2k+1}(\mathbb{R})$ such that on $[-1,1]$, we have
    \begin{equation}
        |x^2 - h_k(x)| \leq |\sigma(x)^2 - h_k(\sigma(x))| + |\sigma(-x)^2 - h_k(\sigma(-x))| \leq 4^{-k}.
    \end{equation}
    (Since one of $\sigma(x)$ and $\sigma(-x)$ is $0$.)
    
    Finally, to construct a network which approximates products, we use the formula
    \begin{equation}
        xy = 2\left(\left(\frac{x+y}{2}\right)^2 - \left(\frac{x}{2}\right)^2 - \left(\frac{y}{2}\right)^2\right)
    \end{equation}
    If $x,y\in [-1,1]$, then all of the terms which are squared in the previous equation are also in $[-1,1]$, so that we can approximate these squares using the network $h_k$. Applying the affine map
    \begin{equation}
        \begin{pmatrix}
            x\\
            y
        \end{pmatrix}\rightarrow \begin{pmatrix}
            (x+y)/2\\
            x/2\\
            y/2
        \end{pmatrix}\in \Upsilon^0(\mathbb{R}^2,\mathbb{R}^3),
    \end{equation}
    then successively applying $h_k$ to the first, second, and third coordinates using Lemmas \ref{select-coordinates-lemma} and \ref{composition-lemma}, and finally applying the affine map
    \begin{equation}
        \begin{pmatrix}
            x\\
            y\\
            z
        \end{pmatrix}\rightarrow 2(x - y - z)\in \Upsilon^0(\mathbb{R}^3),
    \end{equation}
    we obtain a network $f_k\in \Upsilon^{13,6k+3}(\mathbb{R}^2)$ such that for $x,y\in [-1,1]$ we have
    \begin{equation}
        |f_k(x,y) - xy| \leq 6\cdot 4^{-k},
    \end{equation}
    as desired.
\end{proof}
\section{Bit Extraction Network Construction}\label{bit-extraction-appendix}
\begin{proof}[Proof of Proposition \ref{bit-extractor-network-proposition}]
    We begin by noting that for any $\epsilon > 0$ the piecewise linear maps
    \begin{equation}\label{approximate-bit-network-eq}
        b_\epsilon(x) = \begin{cases}
            0 & x \leq 1/2 - \epsilon\\
            \epsilon^{-1}(x - 1/2+\epsilon) & 1/2-\epsilon< x \leq 1/2\\
            1 & x > 1/2
        \end{cases}
    \end{equation}
    and
    \begin{equation}
            g_\epsilon(x) = \begin{cases}
            x & x\leq 1-\epsilon\\
            \frac{1-\epsilon}{\epsilon}(1-x) & 1-\epsilon < x \leq 1\\
            x-1 & x > 1
        \end{cases}
    \end{equation}
    satisfy $b_\epsilon,g_\epsilon\in \Upsilon^{5,2}(\mathbb{R})$ by Proposition \ref{piecewise-linear-function-proposition}. In addition, these functions have been designed so that if $\epsilon < 2^{-n}$, we have for any $x$ of the form \eqref{binary-expansion-x} that
    \begin{equation}
        b_\epsilon(x) = x_1,~g_\epsilon(2x) = 0.x_2x_3\cdots x_n.
    \end{equation}
    We now construct the network $f_{n,m}$ by induction on $m$. In what follows, we assume that all of our inputs $x$ are of the form \eqref{binary-expansion-x}. The base case when $m = 0$ is simply the affine map
    \begin{equation}
        x\rightarrow \begin{pmatrix}
            x\\
            0
        \end{pmatrix}\in \Upsilon^0(\mathbb{R},\mathbb{R}^2).
    \end{equation}
    For the inductive step, we suppose that we have constructed a map
    \begin{equation}
        f_{n,m-1}(x) = \begin{pmatrix}
            0.x_mx_{m+1}\cdots x_n\\
            x_1x_2\cdots x_{m-1}.0
        \end{pmatrix}\in \Upsilon^{9,4(m-1)}(\mathbb{R},\mathbb{R}^2)
    \end{equation}
    We then compose this network with an affine map which doubles and duplicates the first component
    \begin{equation}
        \begin{pmatrix}
            x\\
            y
        \end{pmatrix}\rightarrow \begin{pmatrix}
            2x\\
            x\\
            y
        \end{pmatrix}\in \Upsilon^0(\mathbb{R}^2,\mathbb{R}^3)
    \end{equation}
    to get the map
    \begin{equation}
        x\rightarrow \begin{pmatrix}
            x_m.x_{m+1}\cdots x_n\\
            0.x_mx_{m+1}\cdots x_n\\
            x_1x_2\cdots x_{m-1}.0
        \end{pmatrix}\in \Upsilon^{9,4(m-1)}(\mathbb{R},\mathbb{R}^3).
    \end{equation}
    Next we choose $\epsilon < 2^{-n}$ and use Lemmas \ref{composition-lemma} and \ref{select-coordinates-lemma} to apply $g_\epsilon$ to the first component and then $b_\epsilon$ to the second component. This gives a map
    \begin{equation}
        x\rightarrow \begin{pmatrix}
            0.x_{m+1}\cdots x_n\\
            x_m\\
            x_1x_2\cdots x_{m-1}.0
        \end{pmatrix}\in \Upsilon^{9,4m}(\mathbb{R},\mathbb{R}^3).
    \end{equation}
    Finally, we complete the inductive step by composing with the affine map
    \begin{equation}
        \begin{pmatrix}
            x\\
            y\\
            z
        \end{pmatrix}\rightarrow \begin{pmatrix}
            x\\
            2z + y
        \end{pmatrix}\in \Upsilon^0(\mathbb{R}^3,\mathbb{R}^2).
    \end{equation}
\end{proof}

\begin{proof}[Proof of Lemma \ref{mapping-to-integers-lemma}]
    	Start with the following piecewise linear function 
      \begin{equation}
        g_\epsilon(x) := \begin{cases}
            0 & x \leq 0 \\
            (j + \epsilon^{-1}(x - j/b)) & j/b - \epsilon < x \leq j/b,~\text{for $j=1,...,b-1$}\\
            j & j/b < x \leq (j+1)/b - \epsilon,~\text{for $j=0,...,b-1$}\\
            b-1 & x > 1.
        \end{cases}
      \end{equation}
      Note that this function has $2b-1$ pieces and so by Proposition \ref{piecewise-linear-function-proposition} we have $g_\epsilon\in \Upsilon^{5,2(b-1)}(\mathbb{R})$.
      
      We set $x_0 = x$ and $q_0 = 0$ and consider the following recursion
	\begin{equation}
		x_{n+1} = bx_n - g_\epsilon(x_n),~q_{n+1} = bq_n + g_\epsilon(x_n).
	\end{equation}
	It is easy to verify that if $x_0 = x\in [jb^{-l}, (j+1)b^{-l} - \epsilon)$, then $q_l = j$, since in this case all iterates $x_n\notin \cup_{j=1}^b(j/b - \epsilon, j/b)$ so that $g_\epsilon$ extracts the first bit in the $b$-ary expansion of $x_n$
    $$
        g_\epsilon(x_n) = 
            j~~\text{if}~~ j/b \leq x_n < (j+1)/b.
    $$
    In addition, when $x_0 = x\in [1-b^{-l},1]$, then $q_l = b^l-1$ since $g_\epsilon(x_n) = b-1$ for all $n$.
 
    We implement this recursion using a deep ReLU network as follows. We begin with the affine map
	\begin{equation}
		x\rightarrow \begin{pmatrix}
            x\\
            x\\
            0
        \end{pmatrix} = \begin{pmatrix}
            x_0\\
            x_0\\
            q_0
        \end{pmatrix}\in \Upsilon^0(\mathbb{R},\mathbb{R}^3).
	\end{equation}
	Now, we use induction. Suppose that the map
	\begin{equation}
	x\rightarrow \begin{pmatrix}
            x_n\\
            x_n\\
            q_n
        \end{pmatrix}\in \Upsilon^{9,2(b-1)n}(\mathbb{R},\mathbb{R}^3)
	\end{equation}
	has already been implemented. Then we use Lemmas \ref{composition-lemma} and \ref{select-coordinates-lemma} to apply $g_\epsilon$ to only the second coordinate. This gives the map
	\begin{equation}
	 x \rightarrow \begin{pmatrix}
            x_n\\
            g_\epsilon(x_n)\\
            q_n
        \end{pmatrix}\in \Upsilon^{9,2(b-1)(n+1)}(\mathbb{R},\mathbb{R}^3).
	\end{equation} 
	Finally, we compose with the affine map
	\begin{equation}
		\begin{pmatrix}
            x\\
            y\\
            z
            \end{pmatrix}\rightarrow \begin{pmatrix}
            2(x - y)\\
            2(x - y)\\
            2z + y
            \end{pmatrix}\in \Upsilon^0(\mathbb{R}^3,\mathbb{R}^3)
	\end{equation}
	to complete the inductive step. After $l$ steps of induction, we then compose with the affine map which selects the last coordinate to get the network $q_1\in \Upsilon^{9,2(b-1)l}(\mathbb{R})$.
	
	For higher dimensional cubes $\Omega = [0,1)^d$, we construct an indexing network $q_d\in \Upsilon^{9d,2(b-1)l}(\mathbb{R}^d)$. We use Lemma \ref{concatenating-lemma} to apply $q_1$ to each coordinate of the input.
	Then, we compose with the affine map
	\begin{equation}
		x\rightarrow \sum_{j=1}^d b^{l(j-1)}x_j
	\end{equation}
	to get $q_d\in \Upsilon^{9d,2(b-1)l}(\mathbb{R}^d)$ with
	\begin{equation}
		q_d(\Omega^l_{\textbf{\upshape i},\epsilon}) = \text{\upshape ind}(\textbf{\upshape i}) := \sum_{j=1}^db^{l(j-1)}\textbf{\upshape i}_j.
	\end{equation}
\end{proof}
\end{document}